\documentclass[11pt]{article}

\usepackage[english]{babel}
\usepackage[utf8x]{inputenc}
\usepackage[T1]{fontenc}

\usepackage{subcaption}

\usepackage{geometry}
\usepackage{amsmath,amsfonts,amsthm,amssymb,mathrsfs,dsfont,bbm} 
\usepackage{mathtools}
\usepackage{graphicx}
\usepackage[colorinlistoftodos]{todonotes}
\usepackage[colorlinks=true, allcolors=blue]{hyperref}
\usepackage[capitalize,nameinlink]{cleveref}
\usepackage{verbatim}
\usepackage{algpseudocode,algorithm,algorithmicx}
\allowdisplaybreaks

\newcommand{\bE}{\mathbb{E}}

\newcommand{\Cov}{\mathrm{Cov}}
\newcommand{\Var}{\mathrm{Var}}
\newcommand{\E}{\bE}      

\newcommand{\bone}{\mathbbm{1}}

\DeclareMathOperator{\supp}{supp}

\DeclareMathOperator*{\argmax}{arg\,max}

\newtheorem{theorem}{Theorem}[section]

\newtheorem*{namedtheorem}{\theoremname}
\newcommand{\theoremname}{testing}

\newtheorem{lemma}[theorem]{Lemma}

\newtheorem*{question*}{Question}

\theoremstyle{definition}

\newtheorem{defn}[theorem]{Definition}
\newtheorem{remark}[theorem]{Remark}

\newtheorem{example}[theorem]{Example}
\theoremstyle{plain}

\title{Fast Convergence of Belief Propagation to Global Optima: \\ Beyond Correlation Decay}
\author{Frederic Koehler}
\begin{document}

\maketitle
\begin{abstract}
    Belief propagation is a fundamental message-passing algorithm for probabilistic reasoning and inference in graphical models. While it is known to be exact on trees, in most applications belief propagation is run on graphs with cycles. Understanding
    the behavior of ``loopy'' belief propagation has been a major challenge for researchers in machine learning, and positive convergence results for BP are known under strong assumptions which imply the underlying graphical model exhibits decay of correlations. We show that under a natural initialization, BP converges quickly to the global optimum of the Bethe free energy for Ising models on arbitrary graphs, as long as the Ising model is \emph{ferromagnetic} (i.e. neighbors prefer to be aligned). This holds even though such models can exhibit long range correlations and may have multiple suboptimal BP fixed points. We also show an analogous result for iterating the (naive) mean-field equations; perhaps surprisingly, both results are dimension-free in the sense that a constant number of iterations already provides a good estimate to the Bethe/mean-field free energy.
\end{abstract}

\section{Introduction}
Undirected graphical models, also known as Markov Random Fields, are a general, powerful, and popular framework for modeling and reasoning about high dimensional distributions. These models explain the dependency structure of a probability distribution in terms of interactions along the edges of a (hyper-) graph, which gives rise to a factorization of the joint probability distribution, and the absence of edges in this graph encodes conditional independence relations.

Ising models are a special class of graphical models with a long history and which are popular in applications; they model the interaction of random variables valued in a binary alphabet ($\{\pm 1\}$) 
with exclusively pairwise interactions. Explicitly, the joint pmf of an Ising model is
\begin{equation}\label{eqn:ising-model}
\Pr(X = x) = \exp\left(\frac{1}{2} x^T J x + h \cdot x - \log Z \right) 
\end{equation}
where $x \in \{\pm 1\}^n$, $J : n \times n$ is an arbitrary matrix describing the pairwise interactions between nodes (with zero diagonal), $h$ is an arbitrary vector encoding node biases, and $Z$ is a proportionality constant called the \emph{partition function}. 
Historically, Ising models originated in the statistical physics community as a way to model and study phase transition phenomena in magnetic materials; since then, they have attracted significant interest in the machine learning community and have been applied in a wide variety of domains including finance, social networks, neuroscience, and computer vision (see e.g. references in \cite{li2009markov,mossel2004survey, wainwright-jordan-variational,hinton2012practical}.

Performing sampling and inference on an Ising model is a major computational challenge for which a wide variety of approaches have been developed. One family of methods are Markov-Chain Monte Carlo (MCMC) algorithms, the most popular of which is Gibbs sampling (also known as \emph{Glauber dynamics}) which resamples the spin of one node at a time from its conditional distribution. When run sufficiently long, MCMC methods will draw samples from the true distribution \eqref{eqn:ising-model};
unfortunately, it is well-known in both theory and practice that MCMC methods may become stuck when the probability distribution \eqref{eqn:ising-model} exhibits multi-modal structure; for example, on an $n \times n$ square lattice the Glauber dynamics required exponential time to mix in the low temperature phase \cite{martinelli1999lectures}. 

A popular alternative to markov chain methods are \emph{variational methods}, which typically make some approximation to the distribution \eqref{eqn:ising-model} but often run much faster than MCMC. These methods usually reduce inference on the Ising model to some (typically non-convex) optimization problem, which is solved either by standard optimization methods (e.g. gradient ascent) or by more specialized methods like \emph{message-passing algorithms} (e.g. Belief Propagation). Because of the non-convexity, these methods are typically not guaranteed to return global optimizers of the corresponding variational problem. Indeed, these optimization problems are NP-hard to approximate for general Ising models (see e.g. \cite{jain2018mean} for the case of mean-field approximation).

\emph{Belief propagation} (BP) is a celebrated message passing algorithm which is known to be closely related to the Bethe approximation.
It is a fundamental algorithm for probabilistic inference \cite{Pearl88} which plays a fundamental role in a variety of applications like phylogenetic reconstruction, coding, constraint satisfaction problems, and community detection in the stochastic block model (see e.g. \cite{MezMon:09,mossel2004survey,DKMZ:11}); it is also closely connected to the ``cavity method'' in statistical physics \cite{MezMon:09}. Although BP is observed to works well for many problems, there are few settings on general graphs (i.e. with loops) where it provably works. For example, BP with random initialization is conjectured to achieve optimal reconstruction in the 2-community SBM \cite{DKMZ:11} but no rigorous proof of this result is known.

In this work, we consider two popular variational approximations, the \emph{naive mean-field approximation} and the \emph{Bethe approximation} to the Ising model, and the corresponding heuristic message-passing algorithms which are usually used to solve these optimization problems: mean-field iteration and belief propagation. We show that under a natural and popular assumption of \emph{ferromagneticity} (that $J_{ij} \ge 0$ and $h$ has consistent signs; a.k.a. as an \emph{attractive} graphical model) that these methods do indeed converge to global optimizers of their optimization problems, under a natural initialization, and moreover that their convergence rate is fast and \emph{dimension-free} in the appropriate sense.

\subsection{Background: Variational methods and belief propagation}
We can describe the variational methods we consider in terms of optimization problems
whose goal is to estimate $\Phi := \log Z$, the log partition function or \emph{free energy} of the Ising model. This is a natural to consider because other important quantities can be recovered by differentiating $\log Z$ in some parameter, and because the ability to construct sufficiently precise estimates for $Z$ is equivalent to sampling for a self-reducible family of models \cite{jerrum1986random}.  We note throughout this section we specialize to Ising models, but all of these notions generalize straightforwardly to general Markov random fields (a.k.a. factor models) --- see \cite{MezMon:09} for a more detailed discussion.

The starting point for these variational methods is the \emph{Gibbs variational principle} \cite{MezMon:09} which states
\begin{equation}\label{eqn:variational}
\log Z = \max_{P \in \mathcal{P}(\{\pm 1\}^n)} \E_P[\frac{1}{2} x^T J x + h \cdot x] + H_P(X) 
\end{equation}
where $P$ ranges over probability distributions on $\{\pm 1\}^n$ and $H_P(X)$ is the entropy of $X$ under $P$. This formula is derived by observing the Gibbs measure minimizes the KL divergence to itself and expanding.

The \emph{(naive) mean-field approximation} is given by restricting \eqref{eqn:variational} to product distributions and finding the maximum of the functional
\begin{equation}\label{eqn:mean-field-variational}
\Phi_{MF}(x) := \frac{1}{2} x^T J x + h \cdot x + \sum_i H\left(Ber\left(\frac{1 + x_i}{2}\right)\right)
\end{equation}
where 
$H(Ber(p)) = -p \log p - (1 - p)\log(1 - p)$ is the entropy of a Bernoulli random variable.  Information-theoretically, the optimizer(s) $x^*$ of this optimization problem corresponds to the marginals of a product distribution $\nu$ which minimizes the KL-divergence from the Gibbs measure $\mu$ (defined by \eqref{eqn:ising-model}) to $\nu$. Note that this always gives a lower bound on $\log Z$.

By considering the first-order optimality conditions for \eqref{eqn:mean-field-variational}, one arises at the \emph{mean-field equations}
\begin{equation}\label{eqn:mean-field-eqn}
x = \tanh^{\otimes n}(J \cdot x + h)
\end{equation}
where $\tanh^{\otimes n}$ denotes entry-wise application of $\tanh$. The \emph{mean-field iteration} is the natural iterative algorithm which starts with some $x_0$ and applies \eqref{eqn:mean-field-eqn} iteratively to search for a fixed point.
The error of the mean-field approximation has been extensively studied; the approximation is guaranteed to be accurate when $\|J\|_F = o(n)$ (informally, in unfrustrated models with large average degree); see e.g. \cite{basak2017universality,jain2018mean2,augeri2019transportation}. For example, the recent result of \cite{augeri2019transportation} shows that $|\log Z - \max_x \Phi_{MF}(x)| = O(\sqrt{n} \|J\|_F)$ and the result of \cite{eldan2018} gives even better bounds for some models.

The naive mean-field approximation can be inaccurate on very sparse graphs; the \emph{Bethe approximation} is a more sophisticated approach which has the benefit of being exact on trees \cite{MezMon:09}, and which is always at least as accurate as the mean-field approximation in ferromagnetic models \cite{ruozzi2012bethe}. The Bethe free energy is the maximum of the (typically non-convex) functional
\begin{equation}\label{eqn:bethe-functional}
\Phi_{Bethe}(P) :=  \sum_{i \sim j} J_{ij} \E_{P_{ij}}[X_i X_j] + \sum_i h_i \E_{P_i}[X_i] + \sum_{i \sim j} H_{P_{ij}}(X_i,X_j) - \sum_{i} (\deg(i) - 1) H_{P_i}(X_i) 
\end{equation} 
where $P$ lies in the polytope of locally consistent distributions (equivalently, $SA(2)$ in the Shereli-Adams hierarchy\footnote{The equivalence is given by treating the underlying graph as complete with $J_{ij} = 0$ for unconnected nodes, where the optimal coupling of these non-neighbors is when they are independent.}). Explicitly this polytope is given by constraints:
\begin{align*}
    \sum_{x_i} P_{ij}(x_i,x_j) &= P_j(x_j) & \text{ for all $i,j$ neighbors}\\
    \sum_{x_i} P_i(x_i) &= 1 & \text{ for all $i$}\\
    P_i(x_i) &\ge 0 & \text{ for all $i,x_i$}
\end{align*}
One can derive the Bethe-Peierls (BP) equations from the first-order optimality conditions for this optimization problem; this connection is involved and is discussed further in Section~\ref{sec:background-bethe}. Just like the mean-field equations, the BP equations can be iterated to search for a fixed point, in which case one recovers the belief propagation updates for this setting. Explicitly, for edge messages $\nu_{i \to j}$ with $\nu_{i \to j} \in [-1,1]$ the consistency equation is
\begin{equation}
    \nu_{i \to j} = \tanh(h_i + \sum_{k \in \partial i \setminus j} \tanh^{-1}(\tanh(J_{ik}) \nu_{k \to i}))
\end{equation}
where $\partial i$ denotes the neighborhood of node $i$. Intuitively, this equation describes the expected marginal of node $X_j$ in the absence of the edge between $i$ and $j$. Given $\nu$ which solves these equations, the BP estimate for $\E[X_i]$ can be written as
$\nu_i := \tanh(h_i + \sum_{k \in \partial i} \tanh^{-1}(\tanh(J_{ik}) \nu_{k \to i}))$.
BP also gives an estimate for the free energy in terms of its messages (see equation \eqref{eqn:dual-bethe-defn} in Section~\ref{sec:background-bethe}).

The above derivation of (``loopy'') belief propagation from the the Bethe free energy for general graphical models is due to \cite{YeFrWe:03}. Alternatively, belief propagation can also be derived as the exact solution to computing the partition function of a tree graphical model and this was known much longer ago -- see \cite{Pearl88}. 

\subsection{Our Results}
We analyze the behavior of mean-field iteration and belief propagation in ferromagnetic (a.k.a. attractive) models on arbitrary graphs. We recall their definition below:
\begin{defn}
An Ising model is \emph{ferromagnetic} (with consistent field) if $J_{ij} \ge 0$ for all $i$ and $h_i \ge 0$ for every $i$. (We also can allow $h_i \le 0$ for all $i$, but this is equivalent after flipping signs.)
\end{defn}
We show that in ferromagnetic Ising models, belief propagation and mean-field iteration always converge to the true global optimum of the Bethe free energy and mean-field free energy respectively, as long as they start from the all-ones initialization. Moreover we show that this algorithms converge quickly, giving \emph{linear time} algorithms for estimating the corresponding objective. We note that these results cannot hold for arbitrary Ising models, as even approximating the mean-field free energy is NP-hard in general Ising models with anti-ferromagnetic interactions \cite{jain2018mean}.
\begin{theorem}\label{thm:mean-field-convergence} Fix an arbitrary ferromagnetic Ising model parameterized by $J,h$ and let $x^*$ be a global maximizer of $\Phi_{MF}$. Initializing with $x_0 = \vec{1}$ and defining $x_1,x_2,\ldots$
by iterating the mean-field equations, we have that\footnote{In this theorem and throughout, we use the notation $\|J\|_1,\|J\|_{\infty}$ to refer to the corresponding $\ell_1,\ell_{\infty}$ norms of $J$ when viewed as a vector of entries.} for every $t \ge 1$, 
\[ 0 \le \Phi_{MF}(x^*) - \Phi_{MF}(x_t) \le \min\left\{\frac{\|J\|_1 + \|h\|_1}{t}, 2\left(\frac{\|J\|_1 + \|h\|_1}{t}\right)^{4/3}\right\}. \]
\end{theorem}
\begin{theorem}\label{thm:bp}
Fix an arbitrary ferromagnetic Ising model parameterized by $J,h$, and let $P^*$ be a global maximizer of $\Phi_{Bethe}$.
Initializing $\nu^{(0)}_{i \to j} = 1$ for all $i \sim j$ and performing $t$ steps of the BP iteration on a graph with $m$ edges we have
\[  0 \le \Phi_{Bethe}(P^*) - \Phi^*_{Bethe}(\nu^{(t)}) \le \sqrt{\frac{8mn(1 + \|J\|_{\infty})}{t}} \]
where $\Phi^*_{Bethe}$ is formally defined in \eqref{eqn:dual-bethe-defn}.
\end{theorem}
We also give simple lower bound examples showing these bounds are not too far from optimal; for example, for both algorithms we show the optimal asymptotic rate in $t$ is lower bounded by at least $\Omega(1/t^2)$.  We refer to these bounds as \emph{dimension-independent} because under the typical scaling of the entries of $J$, they show that mean-field iteration/BP achieve good estimates to the variational free energy density after a constant number of iterations. We explain this more precisely in the next remark:
\begin{remark}[Scaling and Dimension-Free Nature]\label{rmk:scaling}
In ferromagnetic models, we usually expect the scaling $\|J\|_1 = \Theta(n),\|h\|_1 = \Theta(n)$ so that all of the terms in the Gibbs variational principle \eqref{eqn:variational} are on the same order, since the entropy scales like $\Theta(n)$ (e.g. the entropy of $Uni(\{\pm 1\}^n)$ is $n \log(2)$). Then the free energy $\log Z$ and its variational approximations all grow like $\Theta(n)$, so when considering the scaling in $n$ one should consider the \emph{free energy density} $\frac{1}{n} \log Z$. Writing the guarantee of Theorem~\ref{thm:mean-field-convergence} for the free energy density when picking the $O(1/t)$ bound, we see
$0 \le \frac{1}{n}\Phi_{MF}(x^*) - \frac{1}{n} \Phi_{MF}(x_t) \le \frac{\|J\|_1 + \|h\|_1}{nt}$
and the rhs is $\Theta(1/t)$ under the assumption $\|J\|_1 = \Theta(n), \|h\|_1 = \Theta(n)$. We get a similar dimension-free guarantee for BP as long as $m = \Theta(n)$, i.e. the model is sparse, and $\|J\|_{\infty} = O(1)$ which is a very rarely violated assumption.
\end{remark}
\begin{remark}[Importance of initialization]
The fast convergence rates we show do not hold for other seemingly natural choices of initialization; e.g. if we start BP with initial messages near zero. For a concrete illustration of this, see Figure~\ref{fig:simple-example-2} in Section~\ref{sec:examples}.
\end{remark}
Finally, we build on ideas developed in our analysis of BP to give a different method, based on convex programming, which has worse dependence on $n$ but converges exponentially fast, i.e. can compute the optimal BP solution to error $\epsilon$ in time $poly(n,\log(1/\epsilon))$. This is described in Appendix~\ref{sec:exp-convergence}; as we explain there, such a method is useful when we care about computing the optimal BP solution accurately in parameter space, as there can be exponentially flat (in terms of $\beta$) directions in the objective.

\textbf{Techniques: } A key theme in the proof of both Theorem~\ref{thm:mean-field-convergence} and Theorem~\ref{thm:bp} is that while the corresponding objective functions are poorly behaved in general, they are well-behaved if we look at the right subset of the space of messages which is induced by the usual partial order structure of vectors in $\mathbb{R}^n$. This lets us take advantage of the monotonicity of both the mean-field and BP updates. In the mean-field case, the objective is even convex on the right subset which makes the analysis particularly clean; however, in the case of BP, the messages do not even live in the same space as $\Phi_{Bethe}$, which is a functional of locally consistent distributions. We overcome this problem by explicitly analyzing the structure of the optimal solution on both the primal (pseudodistribution) and dual (BP message) side, which lets us take advantage of both the smoothness of $\Phi_{Bethe}$ under perturbations and the monotonicity properties of the BP iteration, ultimately enabling us to prove the result.
\subsection{Related Work}
As mentioned above, the general connection between the Bethe free energy and Belief Propagation was established in the work of Yedidia, Freeman and Weiss \cite{YeFrWe:03}. They showed that in any graphical model, the fixed points of BP correspond to the critical points of the Bethe Free Energy. However, their theory by itself does not say anything about the behavior of BP when it is not at a fixed point (as is the case during the BP iteration), or which fixed point (if any) it will converge to. In the special case that the edge constraints are relatively weak, e.g. if they satisfy Dobrushin's uniqueness condition \cite{dobrushin-uniqueness}, one can show that BP converges to a unique fixed point by comparing to what happens on a tree (see \cite{TatikondaJordan:02,mooij-kappen}). BP is also known to converge if the graph has at most one cycle \cite{weiss2000correctness}. Stronger results are known for BP in Gaussian graphical models, in which case BP can be viewed as a method for solving linear equations \cite{weiss-freeman,malioutov2006walk}. 

This work builds upon previous work of Dembo and Montanari \cite{DemboMontanari:10}, who studied the convergence of Belief Propagation in ferromagnetic Ising models with a positive external field strictly bounded away from 0. They showed that in all such models, BP converges at an asymptotically exponential rate to a unique fixed point if initialized with non-negative messages (other fixed points may exist, but they have at least one negative coordinate). From this they derived analytic results for graphs for Ising models which converge locally to (random) trees: these models exhibit \emph{correlation decay}, BP correctly estimates the marginals (e.g. $\E[X_i]$) of the Ising model, and from this derive that the ``cavity prediction'' for the free energy, which is determined by the tree the graph locally converges to, is correct to leading order. 

In contrast, in this work we allow for the complete absence of external field, in which case these models often \emph{do not} exhibit correlation decay and may have multiple fixed points, even in the space of nonnegative messages. Furthermore, the optimal BP fixed point often does not correspond to the true marginals of the underlying Ising model\footnote{For example, if there is no external field then $\E[X_i] = 0$ for all $i$ by symmetry, but e.g. on a 2D lattice at low temperature one can see the optimal BP solution has different marginals \cite{MezMon:09}. In general this kind of behavior correspond to the existence of \emph{phase transitions} at zero external field (for the the corresponding tree  model), which are ruled out in the case of strictly positive external field by the Lee-Yang theorem \cite{lee1952statistical}.}. 
Despite this, we show that BP converges quickly in objective value\footnote{The convergence in parameter space may be slower due to flat directions of the objective: see Section~\ref{sec:examples}. Also, the asymptotic convergence rate is not exponential in some examples. Again, this is new behavior which differs from the strictly positive external field setting.} to the optimal fixed point as long as we start from the all-ones initialization. Another key difference is we are interested in the behavior of BP on general graphs, where the BP estimate cannot always be related to the true free energy; we get around this issue by building on the connection established in \cite{YeFrWe:03} to show that the BP result is always equal to the Bethe free energy on any graph, not necessarily locally tree-like.

A very different line of work studies the dense limit of BP in spin glasses and related models, in the form of the TAP approximation and Approximate Message Passing (AMP); for example, see \cite{bolthausen2014iterative,bayati2011dynamics}. These results are more concerned with dense models with random edge weights and are motivated by CLT-type considerations; the models they consider are typically far from ferromagnetic and thus require in the dense limit the TAP approximation instead of the naive mean field approximation. Since we consider arbitrary graphs instead of (dense) random graphs, these techniques are not applicable.

Outside of message passing algorithms, we note that in ferromagnetic Ising models it is actually possible to sample efficiently from the Boltzmann distribution by using a special Markov chain which performs non-local updates: this was proved in a landmark work of Jerrum and Sinclair \cite{JerrumSinclair:90}. This result can be used to give an algorithm for approximating the mean-field free energy  using a graph blow-up reduction \cite{jain2018mean}. Also, for ferromagnetic models it was shown previously that the Bethe free energy can be computed in polynomial time using discretization with submodularity-based methods in \cite{korvc2012approximating,weller2013bethe}. The polynomials in the runtime guarantees for both methods are fairly large; neither can achieve anywhere near the essentially linear runtime guarantee we give for BP.
\section{Convergence of Mean-Field Iteration}
In this section, we give the proof of Theorem~\ref{thm:mean-field-convergence} by analyzing the mean-field iteration. Organizationally, we split this theorem into two (corresponding to the two seperate bounds implied by the $\min$): we prove the first bound in the theorem as Theorem~\ref{thm:mean-field-convergence-main} and the second $O(1/t^{4/3})$ bound as Theorem~\ref{thm:mean-field-convergence-secondary}. 
\subsection{Main convergence bound}
In this section we prove the first ($O(1/t)$) bound appearing in Theorem~\ref{thm:mean-field-convergence}, the bound which is better for small $t$; we consider this to be the more significant bound because it gives a meaningful convergence result even when $t = O(1)$ (see Remark~\ref{rmk:scaling}).
A key observation in the proof is that the functional $\Phi_{MF}$ is actually convex on a certain subset of the space of product distributions, and that the iteration stays in this region because the iteration is monotone w.r.t. the partial order structure; this allows us to show progress at each step. 

For the analysis of mean-field iteration, it will be very helpful to split the updates up into two steps:
\begin{align*}
y_{t + 1} &:= J x_t + h\\
x_{t + 1} &:= \tanh^{\otimes n}(y_{t + 1}).
\end{align*}
\begin{lemma}
There exists at most one critical point of $\Phi_{MF}$ in $(0,1]^n$.
\end{lemma}
\begin{proof}
Suppose there exist two critical points $y$ and $z$. Recall that being a critical point is equivalent to solving the mean-field equation
\[ y = \tanh^{\otimes n}(J y + h). \]
Consider the line through $y$ and $z$; this line intersects the boundary region $[0,1]^n \setminus (0,1]^n$ at some point; we parameterize the line as $x(t)$ so that $x(0)$ is on this boundary, i.e. $x(0)_i = 0$ for some $i$, $x(t_1) = y$ and $x(t_2) = z$. Without loss of generality we assume that $t_1 < t_2$. Now we consider the behavior of the function
\[ g(t) := \tanh(J_i \cdot x(t) + h_i) - x(t)_i \]
on this line. Observe that by definition $g(0) =  \tanh(J_i \cdot x(0) + h_i) - 0 \ge 0$ and $g(t_1) = 0$. It follows from strict concavity that $g(t_2) < 0$ since $t_2 > t_1$, so $z$ cannot be a fixed point, which gives a contradiction.
\end{proof}
Based on this lemma, we define $x^*$ to be the global maximizer of $\Phi_{MF}$ in $[0,1]^n$. Define $S := \{x \in (0,1]^n : x_i \ge x^*_i \}$. 
\begin{lemma}\label{lem:concavity}
The mean-field free energy functional $\Phi_{MF}$ is concave on $S$.
\end{lemma}
\begin{proof}
First we claim that $\Phi_{MF}$ is concave at $x^*$. If $x^*$ is on the interior of $[0,1]^n$, then this follows from the second-order optimality condition. From the mean-field equations (first order optimality condition) we see that it's impossible that there are any coordinates such that $x^*_i = 1$, and that if the graph is connected and there is a single coordinate such that $x^*_i = 0$, that the entire vector $x^* = 0$. If $x^* = 0$, then the maximum eigenvalue of $J$ must be $1$, so the free energy functional is globally concave -- otherwise, by the Perron-Frobenius theorem there exists a eigenvector of $J$ with all nonnegative entries and with eigenvalue greater than $1$, from which we see that $x^* = 0$ cannot be the global optimum.

Now, it is easy to see that $\Phi_{MF}$ is concave on all of $S$, becuase if $0 \le x \le y$ coordinate-wise then $\nabla^2 \Phi_{MF}(x) \succeq \nabla^2 \Phi_{MF}(y)$, which follows because 
\[ \nabla^2 \Phi_{MF}(x) - \nabla^2 \Phi_{MF}(y) = (1/4)\sum_i (H''((1 + x)/2) - H''((1 + y)/2)) e_i e_i^T \succeq 0. \]
since $H''((1 + x)/2) = \frac{-2}{1 - x^2}$.
\end{proof}
\begin{theorem}[Main bound in Theorem~\ref{thm:mean-field-convergence}]\label{thm:mean-field-convergence-main}
Suppose that $x_0 \in S$ and define $(x_t,y_t)_{t = 1}^{\infty}$ by iterating the mean-field equations. Then for every $t$, $x_t \in S$. Furthermore
\[ \Phi_{MF}(x^*) - \Phi_{MF}(x_t) \le \frac{\|J\|_1 + \|h\|_1}{t}. \]
\end{theorem}
\begin{proof}
To show that $x_t \in S$, observe that the mean-field iteration is monotone: if $x \le x'$, then $\tanh^{\otimes n}(J x + h) \le \tanh^{\otimes n}(J x' + h)$. Therefore, because $x^* \le x_0$ we see that $x^* = \tanh^{\otimes n}(J x^* + h) \le \tanh^{\otimes n}(J x_0 + h) = x_1$ and so on iteratively. 

To prove the convergence bound, first note that
\[ \frac{\partial}{\partial x_i} \Phi_{MF}(x) = J_i \cdot x + h_i - \tanh^{-1}(x_i) \]
and then observe by Lemma~\ref{lem:concavity} and concavity that
\begin{align*}
\Phi_{MF}(x^*) - \Phi_{MF}(x_t) 
\le \langle \nabla \Phi_{MF}(x_t), x^* - x_t \rangle
\le \|\nabla \Phi_{MF}(x_t)\|_1
&= \sum_i |\tanh^{-1}(x_{t,i}) -(Jx_t + h)_i|  \\
&= \sum_i y_{t,i}-y_{t + 1,i}
\end{align*}
where the second inequality was by H\"older's inequality and $\|x^* - x_t\|_{\infty} \le 1$, and the last equality follows from the definition of $y_t$ and because $y_{t + 1} \le y_t$ coordinate-wise. We can also see that $\Phi_{MF}(x_t)$ is a monotonically increasing function of $t$ by considering the path between $x_t$ and $x_{t + 1}$ which updates one coordinate at a time, as the gradient always has non-positive entries along this path. Therefore if we sum over $t$ we find that
\[ \Phi_{MF}(x^*) - \Phi_{MF}(x_{T}) \le \frac{1}{T} \sum_{t = 1}^{T} (\Phi_{MF}(x^*) - \Phi_{MF}(x_t)) \le \frac{1}{T} \sum_{i = 1}^n (y_{1,i} - y_{T + 1,i}) \le \frac{\|J\|_1 + \|h\|_1}{T} \]
since $y_{T + 1,i} \ge 0$ and $y_{1,i} \le \sum_j J_{ij} + h_i \le \|J_i\|_1 + h_i$.
\end{proof}
The following simple example shows that the above result is not too far from optimal, in the sense that an asymptotic rate of $o(1/t^2)$ is impossible. We take advantage of the fact that when the model is completely symmetrical, the behavior of the update can be reduced to a 1-dimensional recursion, which is a very standard trick (see e.g. \cite{MezMon:09,parisi1988statistical}).
\begin{example}\label{example:d-regular-mf}
Consider any $d$-regular graph with no external field and edge weight $\beta = 1/d$, which corresponds to the naive mean field prediction for the critical temperature. By symmetry, analyzing the mean field iteration reduces to the 1d recursion $x \mapsto \tanh(x)$ which behaves like $x \mapsto x - x^3/3$ near the fixed point $x = 0$. Solving this recurrence, we see that $x$ converges to $0$ at rate $\Theta(1/\sqrt{t})$. In terms of $x$, the estimated mean field free energy is $(n/2) x^2 + n H(\frac{1 + x}{2})$, so by expanding we see that the estimated free energy converges at a $\Theta(1/t^2)$ rate in this example.
\end{example}

\subsection{A Faster Asymptotic Rate}\label{sec:asymptotic-improvement}
The above theorem and lower bound leave a gap between $O(1/t)$ and $\Omega(1/t^2)$ for the asymptotic rate of the mean-field iteration. This section is devoted to showing that for large $t$, we can obtain an improved asymptotic rate of $O(1/t^{4/3})$ for the mean-field iteration using a slightly more involved variant of the argument from the previous section. The key insight is that we can obtain some control of $\|x - x^*\|_{\infty}$ by consider the behavior of higher-order terms when expanding around $x^*$, and this can be used to get better bounds on the convergence in objective. 
\begin{lemma}\label{lem:grad-lbound-asymptotic}
Suppose that $x \in S$. Then
\[ \|\nabla \Phi_{MF}(x)\|_1 \ge \frac{\|x - x^*\|_4^4}{\|x - x^*\|_{\infty}} \]
where $x^*$ is as above, the global maximizer of $\Phi_{MF}$ in $[0,1]^n$.
\end{lemma}
\begin{proof}
Recall that
\[ \nabla \Phi_{MF}(x) = Jx + h - \sum_i \tanh^{-1}(x_i) e_i. \]
Since $x^*$ is a critical point and local maximum, so $\nabla \Phi_{MF}(x^*) = 0$ and $\nabla^2 \Phi_{MF}(x^*) \preceq 0$, then using that $\frac{d^2}{dx^2} \tanh^{-1}(x) = \frac{2x}{(1 - x^2)^2}$, we see that by applying the fundamental theorem of calculus twice that
\[ \nabla \Phi_{MF}(x) = J (x - x^*) - \sum_i e_i (\tanh^{-1}(x_i) - \tanh^{-1}(x^*_i)) = \nabla^2 \Phi_{MF}(x^*) (x - x^*) - \sum_i e_i \int_{x^*_i}^{x_i} \int_{x^*_i}^{z} \frac{2y}{(1 - y^2)^2} dy dz  \]
and so
\begin{align*}
\langle x^* - x, \nabla \Phi_{MF}(x) \rangle 
&\ge \sum_i (x_i - x^*_i) \int_{x^*_i}^{x_i} \int_{x^*_i}^{z} \frac{2y}{(1 - y^2)^2} dy dz  \\
&\ge \sum_i (x_i - x^*_i) \int_{x^*_i}^{x_i} \int_{x^*_i}^{z} 2y dy dz \\
&= \sum_i (x_i - x^*_i) (x_i^3/3 - (x_i^*)^3/3 - (x_i - x^*_i)(x^*_i)^2) \\
&= \sum_i (x_i - x^*_i)^2 (x_i^2 + x_i x_i^*) \ge \sum_i (x_i - x^*_i)^4
\end{align*} 
where in the last inequality we used $x_i \ge x^*_i \ge 0$. Finally the result follows combining the above with $\langle x^* - x, \nabla \Phi_{MF}(x) \rangle \le \|x^* - x\|_{\infty}\|\nabla \Phi_{MF}\|_1$ by H\"older's inequality.
\end{proof}
\begin{theorem}[Second bound in Theorem~\ref{thm:mean-field-convergence}]\label{thm:mean-field-convergence-secondary}
Suppose that $x_0 \in S$ and define $(x_t,y_t)_{t = 1}^{\infty}$ by iterating the mean-field equations. Then for every $t$, $x_t \in S$. Furthermore for any $t \ge 1$,
\[ \|x_t - x^*\|_{\infty}^3 \le \frac{\|J\|_1 + \|h\|_1}{t}\]
and
\[ \Phi_{MF}(x^*) - \Phi_{MF}(x_{2t}) \le \left(\frac{\|J\|_1 + \|h\|_1}{t}\right)^{4/3}. \]
\end{theorem}
\begin{proof}
From Lemma~\ref{lem:grad-lbound-asymptotic} we see that
\[ \|x - x^*\|_{\infty}^3 \le \frac{\|x - x^*\|_4^4}{\|x - x^*\|_{\infty}} \le \|\nabla \Phi_{MF}(x)\|_1   \]
and so as in the proof of Theorem~\ref{thm:mean-field-convergence-main} we see that for any $T$,
\[ \|x_T - x^*\|^3_{\infty} \le \frac{1}{T} \sum_{t = 1}^T \|x_t - x^*\|_{\infty}^3 \le \frac{1}{T} \sum_{t = 1}^T\|\nabla \Phi_{MF}(x_t)\|_1 = \frac{1}{T} \sum_{i = 1}^n (y_{1,i} - y_{T + 1,i}) = \frac{\|J\|_1 + \|h\|_1}{T}.  \]
Therefore for any $t' > T$ we see by convexity and H\"older's inequality
\begin{align*}
\Phi_{MF}(x^*) - \Phi_{MF}(x_t) 
\le \langle \nabla \Phi_{MF}(x_t), x^* - x_t \rangle 
&\le \left(\frac{\|J\|_1 + \|h\|_1}{T}\right)^{1/3} \|\nabla \Phi_{MF}(x_t)\|_1 \\
&= \left(\frac{\|J\|_1 + \|h\|_1}{T}\right)^{1/3} \sum_i |\tanh^{-1}(x_{t,i}) -(Jx_t + h)_i|  \\
&= \left(\frac{\|J\|_1 + \|h\|_1}{T}\right)^{1/3} \sum_i (y_{t,i}-y_{t + 1,i})
\end{align*}
and summing this over $t' = T + 1$ to $2T$ and telescoping we see that
\begin{align*} 
\Phi_{MF}(x^*) - \Phi_{MF}(x_{2T}) 
\le \frac{1}{T} \sum_{t' = T + 1}^{2T} (\Phi_{MF}(x^*) - \Phi_{MF}(x_{t'}))
&\le \left(\frac{\|J\|_1 + \|h\|_1}{T}\right)^{1/3}\sum_{i} (y_{T,i} - y_{2T,i}) \\
&\le \left(\frac{\|J\|_1 + \|h\|_1}{T}\right)^{4/3}
\end{align*}
which proves the result.
\end{proof}
\subsection{Aside: Computing the Mean-Field Optimum given Inconsistent Fields}
In this section we describe an efficient algorithm to compute the optimal mean-field approximation even in the situation when the external fields have inconsistent signs (i.e. some of the $h_i$ are negative, some are positive). This is by reduction to the following algorithmic result of \cite{schlesinger2006transforming}, following the same strategy as \cite{korvc2012approximating,weller2013bethe} for the case of the Bethe free energy. We include this result as we were not aware of it appearing explicitly in the literature, though it is not difficult. These results follow the tradition of a long line of work in reducing optimization problems on graphs to submodular minimization and graph min-cut problems.
\begin{theorem}[\cite{schlesinger2006transforming}]
Let $(\Sigma,\le)$ be a finite alphabet equipped with a total ordering, fix a finite graph $G = (V,E)$ and fix functions $f_v : \Sigma \to \mathbb{R}$ and $f_{u,v} : \Sigma \times \Sigma \to \mathbb{R}$. Suppose that every $f = f_{u,v}$ satisfies the following submodularity condition:
\[ f(\min(x_1,x_2),\min(y_1,y_2)) + f(\max(x_1,x_2),\max(y_1,y_2)) \le f(x_1,y_1) + f(x_2,y_2) \]
Then the optimization problem
\[ \min_{L : V \to \Sigma} \left[\sum_{v \in V} f_v(L(v)) + \sum_{u \sim v} f_{u,v}(L(u),L(v)) \right]. \]
is efficiently solvable in time $poly(|\Sigma|,|V|)$.
\end{theorem}

Supermodularity (which will become submodularity after converting to a minimization problem by negating) for the edge interactions $J_{ij} x_i x_j$ is immediate, so to apply this theorem all we need to do is discretize the optimization problem $\max \Phi_{MF}$ appropriately: to do this we compute Lipschitz constants on the relevant part of the space.
First observe that
\[ |x^*_i| = |\tanh(J_i \cdot x^*_{\sim i} + h_i)| \le \tanh(\|J_i\|_1 + |h_i|) \]
so if we restrict $x_i$ to lie within $[-\tanh(\|J_i\|_1 + |h_i|), \tanh(\|J_i\|_1 + |h_i|)$ then
\[ \max_{x_i} |\frac{d}{d x_i}H(Ber(\frac{1 + x_i}{2})| = \max_{x_i} \tanh^{-1}(\tanh(\|J_i\|_1 + |h_i|)) = |J_i|_1 + |h_i| \]
so $H(x_i)$ is $(\|J_i\|_1 + |h_i|)$-Lipschitz on this interval. Similarly we observe that
\[ \left|\sum_j J_{ij} x_j + h_i\right| \le \|J_i\|_1 + |h_i| \]
so if we discretize each coordinate with grid size $\frac{\epsilon}{2(\|J_i\|_1 + |h_i|)}$ then we change the objective by at most $\epsilon n$. Then by the result of \cite{schlesinger2006transforming} this problem can be solved to optimality in time $poly(1/\epsilon, n, \max_i \|J_i\| + |h_i|)$. Formally, this shows the following result:
\begin{theorem}
Fix an Ising model with ferromagnetic interactions ($J_{ij} \ge 0$) and arbitrary (not necessarily consistent) external field $h$. Then the mean-field free energy $\max_x \Phi_{MF}(x)$ can be approximated within error $\epsilon n$ in time $poly(1/\epsilon, n, \|J\|_1,\|h\|_1)$.
\end{theorem}
\section{Rapid Convergence of Belief Propagation}
In this section, we give the proof of our main result Theorem~\ref{thm:bp} by analyzing BP. This proof is considerably more involved than the case of the mean-field iteration; a major conceptual difference between the two iterations is that the mean-field iteration always maintains a valid product distribution, and so can be understood in terms of the landscape of $\Phi_{MF}$, whereas BP operates on ``dual'' variables which do not correspond to valid pseudodistributions except at fixed points, so analyzing the landscape of $\Phi_{Bethe}$ by itself does not suffice. We get around this by also considering the behavior of a dual functional $\Phi^*_{Bethe}$ which is well-defined for every set of BP messages; however this functional is poorly behaved in general (it can be unbounded and its critical points are typically saddle points). We are able to handle these difficulties by identifying the special behavior of BP and $\Phi^*_{Bethe}$ on two special subsets of the BP messages arising from the partial order structure: the \emph{prefixpoints} and \emph{postfixpoints}. Finally, when analyzing BP in these regions we are able to relate in a useful way its behavior at different values of external field, enabling us to use a convexity argument from \cite{DemboMontanari:10} which cannot be directly applied to our setting.
\subsection{Background: BP and the Bethe Free Energy}\label{sec:background-bethe}
In this section we recall the necessary facts we need about the relationship  between the Bethe free energy and BP. This relationship and corresponding formulas are a bit involved so we sketch the derivation in Appendix~\ref{apdx:background-bethe-deferred}.

It was shown in \cite{YeFrWe:03} that by writing down the Lagrangian corresponding to the optimization problem \eqref{eqn:bethe-functional} over the polytope of locally consistent distributions, one can derive an expression for the Bethe free energy at a critical point of the Lagrangian in terms of the dual variables (Lagrange multipliers). 
After a change of variables to $\nu$, this lets us define
(for all $\nu$, not necessarily fixed points of the BP equations), the \emph{dual Bethe free energy}
\begin{equation}\label{eqn:dual-bethe-defn}
    \Phi_{Bethe}^*(\nu) := \sum_i F_i(\nu) - \sum_{i \sim j} F_{ij}(\nu). 
\end{equation} 
where
\begin{align*} 
F_i(\nu) 
= \log \left[e^{h_i} \prod_{j \in \partial i} (1 + \tanh(J_{ij}) \nu_{j \to i})  + e^{-h_i} \prod_{j \in \partial i} (1 - \tanh(J_{ij}) \nu_{j \to i}) \right] + \sum_{j \in \partial i} \log \cosh(J_{ij})
\end{align*}
and
$F_{ij}(\nu) 
= \log(1 + \tanh(J_{ij}) \nu_{i \to j} \nu_{j \to i}) + \log\cosh(J_{ij})$.
We remark (see \cite{MezMon:09}) that in the case the graph is a tree, it's known that the Bethe free energy is a convex function, so $\Phi_{Bethe}^*(\nu)$ plus the Lagrange multiplier terms is actually the Lagrangian dual and thus has a natural interpretation for all $\nu$. This is not true in general, however we will soon see that $\Phi_{Bethe}^*$ does have useful properties on some special subsets of the space of messages.
\subsection{Optimization Landscape}
The following lemma establishes that $\phi(\nu)_{i \to j}$ is a concave monotone function for nonnegative $\nu$. 
\begin{lemma}\label{lem:monotone-concave}
Suppose that $f(x) = \tanh(h + \sum_i \tanh^{-1}(x_i))$ for any $h \ge 0$.
Then $f$ is a concave monotone function on the domain $[0,1)^n$. Furthermore $\nabla^2 f(x) \prec 0$ unless $h = 0$ and $|\supp(x)| \le 1$.
\end{lemma}
\begin{proof}
Observe that 
\[ \frac{\partial f}{\partial x_i}(x) = \frac{1 - f(x)^2}{1 - x_i^2} \ge 0\]
which proves monotonicity, and
\[ \frac{\partial^2 f}{\partial x_i \partial x_j}(x) = \frac{-2f(x)(1 - f(x)^2)}{(1 - x_i^2)(1 - x_j^2)} + \bone(i = j) \frac{(1 - f(x)^2) 2x_i}{(1 - x_i^2)^2} = \frac{(2x_i \bone(i = j) - 2f(x))(1 - f(x)^2)}{(1 - x_i^2)(1 - x_j^2)}.  \]
Note that for any vector $w$, if $w'_i = (1 - f(x)^2) w_i/(1 - x_i^2)$ then
\[ \sum_{ij} w_i \frac{\partial^2 f}{\partial x_i \partial x_j}(x) w_j = \sum_{ij} w'_i (2 x_i \bone(i = j) - 2f(x)) w'_j = \sum_i 2x_i(w'_i)^2 - 2f(x) (\sum_i w'_i)^2 \le 0 \]
since $x_i \le f(x)$, which proves concavity. If $h > 0$ or $|\supp(x)| \ge 2$ then this is a strictly inequality since $x_i < f(x)$.
\end{proof}
There are two special subsets of the nonnegative messages which will play key roles in our analysis. They are the set of \emph{pre-fixpoints} and \emph{post-fixpoints} (following standard poset terminology), defined by
\[ S_{pre} = \{\nu : 0 \le \phi(\nu)_{i \to j} \le \nu_{i \to j} \}, \quad S_{post} = \{\nu : 0 \le \nu_{i \to j} \le \phi(\nu)_{i \to j} \}. \]
Note that both sets contain the nonnegative fixed points; also
we note from Lemma~\ref{lem:monotone-concave} that $S_{post}$ is a convex set, whereas $S_{pre}$ is typically non-convex and even disconnected. 
The gradient of $\Phi^*_{Bethe}$ is well-behaved on these sets:
\begin{lemma}\label{lem:dual-gradient-properties}
For any $\nu \ge 0$, $\| \nabla \Phi^*_{Bethe}(\nu)\|_{\infty} \le 1$. Furthermore, if $\nu \in S_{pre}$ then $\nabla \Phi^*_{Bethe}(\nu) \le 0$ and if $\nu \in S_{post}$ then $\nabla \Phi^*_{Bethe}(\nu) \ge 0$.
\end{lemma}
\begin{proof}
This will follow once we compute the gradient of $\Phi^*_{Bethe}(\nu)$. Recall that we defined $\theta_{ij} = \tanh(J_{ij})$.
Observe that
\begin{align}\label{eqn:dual-gradient} \frac{\partial\Phi^*_{Bethe}}{\partial_{\nu_{i \to j}}} (\nu) 
&= \frac{\partial F_j}{\partial \nu_{i \to j}} - \frac{\partial F_{ij}}{\partial \nu_{i \to j}}  \nonumber \\
&= \frac{e^{h_j} \theta_{ij} \prod_{k \in \partial j \setminus i} (1 + \theta_{jk} \nu_{k \to j}) - e^{-h_j} \theta_{ij} \prod_{k \in \partial j \setminus i} (1 - \theta_{jk} \nu_{k \to j})}{e^{h_j} \prod_{k \in \partial j} (1 + \theta_{jk} \nu_{k \to j}) + e^{-h_j} \prod_{k \in \partial j} (1 - \theta_{jk} \nu_{k \to j})} - \frac{\theta_{ij} \nu_{j \to i}}{1 + \theta_{ij} \nu_{i \to j} \nu_{j \to i}} \nonumber \\
&=\frac{1}{\nu_{i \to j} + 1/(\theta_{ij} \phi(\nu)_{j \to i})} - \frac{1}{\nu_{i \to j} + 1/(\theta_{ij} \nu_{j \to i})}
\end{align}
where (as defined earlier) $\phi(\nu)_{j \to i}$ denotes the next BP message from $j$ to $i$ based on the current $\nu$. As long as $\nu \ge 0$, we see that
\begin{align*} 
\left|\frac{1}{\nu_{i \to j} + 1/(\theta_{ij} \phi(\nu)_{j \to i})} - \frac{1}{\nu_{i \to j} + 1/(\theta_{ij} \nu_{j \to i})}\right| 
&= \left|\int^{1/(\theta_{ij} \nu_{j \to i})}_{1/(\theta_{ij} \phi(\nu)_{j \to i})} \frac{1}{(\nu_{i \to j} + x)^2} dx\right| \\
&\le \left|\int^{1/(\theta_{ij} \nu_{j \to i})}_{1/(\theta_{ij} \phi(\nu)_{j \to i})} \frac{1}{x^2} dx\right|
= \theta_{ij}|\nu_{j \to i} - \phi(\nu)_{j \to i}| \le 1
\end{align*}
which proves that $\|\nabla \Phi^*_{Bethe}(\nu)\|_{\infty} \le 1$. If $\nu \in S_{pre}$ or $S_{post}$ then the signs are determined by \eqref{eqn:dual-gradient} as claimed.
\end{proof}
The Knaster-Tarski theorem \cite{tarski1955lattice} shows that the fixed points of $\phi$ must form a complete lattice, and in particular shows that a greatest fixed point must exist; the following lemma identifies it explicitly.
\begin{lemma}\label{lem:nu*}
Suppose that BP is run from initial messages $\nu^{(0)}_{i \to j} = 1$. The messages converge to a fixed point $\nu^*$ of the BP equations such that for any other fixed point $\mu$, $\mu_{i \to j} \le \nu^*_{i \to j}$. Furthermore
\[ \Phi^*_{Bethe}(\nu^*) = \max_{\nu \in S_{post}} \Phi^*_{Bethe}(\nu) \]
\end{lemma}
\begin{proof}
If $\nu_{i \to j}^{0)}$ and $\nu^{(t)} := \phi(\nu^{(t - 1)}$ then from monotonicity of $\phi$ (see Lemma~\ref{lem:monotone-concave}) we see this is a coordinate-wise decreasing sequence, which must converge to some fixed point. By monotonicity and induction we also see that for any fixed point $\mu$, $\mu_{i \to j} \le \nu^{(t)}_{i \to j}$ for all $t$, hence for $\nu^*$ as well. Finally, consider any other point $\nu \in S_{post}$: by convexity of $S_{post}$ we see that the line segment from $\nu$ to $\nu^*$ is entirely contained in $S_{post}$, by Lemma~\ref{lem:dual-gradient-properties} we see that for any $x$ on this interpolating line that $\nabla \Phi^*_{Bethe}(x) \cdot (\nu^* - \nu) \ge 0$, and integrating this gives that $\Phi^*_{Bethe}(\nu) \le \Phi^*_{Bethe}(\nu^*)$ as desired.
\end{proof}
As our final preparation for the theorem, we establish that at least one optimal BP fixed point has only nonnegative messages. First we need the following technical lemma, which allows us to reason about the behavior of the optimal couplings in the Bethe approximation. The realization that solving for this coupling analytically is feasible is due originally to \cite{welling-teh}, although we parameterize the solution differently.
\begin{lemma}\label{lem:recoupling}
Suppose that $\E[X] \ge 0$ and $\E[Y] \ge 0$ where $X,Y$ are valued in $\{\pm1\}$. Then
\[ \max_{coupling} [-\beta\Cov(X,Y) + H(X,Y)] \le \max_{coupling} [\beta \Cov(X,Y) + H(X,Y)]\]
where the maximum ranges over all couplings (i.e. possible joint distributions) $P$ of $X$ and $Y$.
\end{lemma}
\begin{proof}
By subtracting $H(Y)$ on both sides we reduce to showing
\begin{equation}\label{eqn:coupling-goal} \max_{coupling} [-\beta\Cov(X,Y) + H(X|Y)] \le \max_{coupling} [\beta \Cov(X,Y) + H(X|Y)].
\end{equation}
We will do this by showing both sides are differentiable w.r.t. $\beta$ and that the derivative of the rhs ($\Cov(X,Y)$ at the optimal coupling for the rhs) is larger than the derivative of the lhs ($-\Cov(X,Y)$ for the optimal coupling for the lhs), so that integrating gives the desired inequality.

First we characterize the optimizer of the rhs of \eqref{eqn:coupling-goal}. Let $\rho := \Cov(X,\frac{Y}{\Var(Y)})$. Then $\E[X | Y] = \E[X] + \rho (Y - \E Y)$ since $\Cov(\E[X | Y], Y) = \Cov(X,Y) = \rho$.
Thus the objective maximized by the rhs is 
\[ f(\rho) := \beta \Var(Y) \rho + \E_Y H(\E[X] + \rho (Y - \E Y)). \]
where $H(x) := H(Ber(\frac{1 + x}{2}))$.
Differentiating in $\rho$ we see that the optimum is when
\[ \Var(Y)\beta = \E_Y[(Y - \E Y)\tanh^{-1}(\E[X] + \rho (Y - \E Y))] = \Cov(Y, \tanh^{-1}(\E[X] + \rho(Y - \E Y))).\]
Let this relation define $\rho(\beta) \ge 0$. 
Similarly, define $\rho'(\beta) \ge 0$ to be the solution to
$-\Var(Y)\beta = \Cov(Y, \tanh^{-1}(\E[X] - \rho'(Y - \E Y)))$
i.e.
\[ \Var(Y)\beta = \Cov(Y, \tanh^{-1}(-\E[X] + \rho'(Y - \E Y)))\]
We now claim that $\rho(\beta) \ge \rho'(\beta)$. As previously described, if we show this then by integrating w.r.t. $\beta$ we get the final inequality. 
To prove the claim, first 
subtract the above terms to get that
\[ 0 = \Cov[Y, \tanh^{-1}(\E[X] + \rho(Y - \E Y)) - \tanh^{-1}(-\E[X] + \rho'(Y - \E Y))]. \]
Since the covariance is 0 and $Y$ takes on only two values, it means that the rhs is independent of $Y$, therefore
\[  \tanh^{-1}(\E[X] + \rho(1 - \E Y)) - \tanh^{-1}(-\E[X] + \rho'(1 - \E Y)) =  \tanh^{-1}(\E[X] + \rho(-1 - \E Y)) - \tanh^{-1}(-\E[X] + \rho'(-1 - \E Y))\]
and rearranging we get
\[ \tanh^{-1}(\E[X] + \rho(1 - \E Y)) - \tanh^{-1}(\E[X] + \rho(-1 - \E Y)) 
= \tanh^{-1}(-\E[X] + \rho'(1 - \E Y)) -\tanh^{-1}(-\E[X] + \rho'(-1 - \E Y)) \]

Define $g(x,r) := \tanh^{-1}(x + r) - \tanh^{-1}(x - r)$ for $x,r$ s.t. $|x| + |r| < 1$ and $r \ge 0$.
Then the above equation says $g(\E[X] - \rho \E[Y], \rho) = g(-\E[X] - \rho' \E[Y], \rho')$. We obsere that $g$ is even, strictly increasing in $r$, and strictly increasing in $x$ for $x \ge 0$ since $\frac{\partial}{\partial x} g(x,r) = \frac{1}{1 - (x + r)^2}- \frac{1}{1 - (x - r)^2} \ge 0$.
Since $g$ is an even function, we have 
\begin{equation}\label{eqn:g} g(|\E[X] - \rho \E[Y]|, \rho) = g(|\E[X] + \rho' \E[Y]|, \rho'). \end{equation} Suppose $\rho < \rho'$, then because $g$ is a strictly increasing function in both $x \ge 0$ and $r$ we see that the lhs of \eqref{eqn:g} is strictly less than the rhs, which is a contradiction. Therefore $\rho \ge \rho'$.
\end{proof}

\begin{lemma}\label{lem:global-maximizer}
There exists a BP fixed point in $[0,1)^n$ which corresponds to a global maximizer of the Bethe free energy.
\end{lemma}
\begin{proof}

Observe that for a locally consistent distribution $P$, the Bethe free energy can be rewritten to give
\[ \Phi_{Bethe}(P) = \frac{1}{2} \E[X]^T J \E[X] + \sum_i h_i \E[X_i] + \sum_i H(X_i) + \sum_{i \sim j} (J_{ij} \Cov(X_i,X_j) + H(X_i,X_j) - H(X_i) - H(X_j)).
\]
We first claim that there exists a global maximizer of this functional (over all locally consistent distributions) satisfying $\E[X_i] \ge 0$ for all $i$. To see this, we consider a fixed feasible local distribution $P$ and claim that there exists $Q$ with sign-flipped marginals $\E_Q[X_i] = |\E_P[X_i]|$ and no smaller value of $J_{ij} \Cov(X_i,X_j) + H(X_i,X_j)$. We now describe the couplings along each edge $i \sim j$: if neither or both of $i$ and $j$ were sign-flipped, then we can simply use the same/sign-flipped coupling from before. Now suppose (w.l.o.g.) that $j$ has the same marginal and $i$ was sign-flipped. Then it follows immediately from Lemma~\ref{lem:recoupling} that there exists a coupling $Q_{ij}$ between $X_i$ and $X_j$ s.t. $J_{ij} \Cov_{Q_{ij}}(X_i,X_j) + H_{Q_{ij}}(X_i,X_j) \ge J_{ij} \Cov_P(X_i,X_j) + H_P(X_i,X_j)$. 

Recall that at a critical point of the Lagrangian, for any edge $i \sim j$
\[ P_{ij}(x_i,x_j) \propto e^{J_{ij} x_i x_j + \lambda'_{i \to j} x_i + \lambda'_{j \to i} x_j}. \]
Since we have shown that there exists a global maximizer $P$ such that $\E[X_i],\E[X_j] \ge 0$ it must be that at least one of $\lambda'_{i \to j},\lambda'_{j \to i} \ge 0$. We now show that there exists another locally consistent distribution $Q$ with $\Phi_{Bethe}(Q) \ge \Phi_{Bethe}(P)$ and with corresponding $\lambda'_{i \to j},\lambda'_{j \to i} \ge 0$ for all edges $i$ and $j$.

The construction of $Q$ goes through the dual free energy $\Phi^*_{Bethe}$.
First recall that $\Phi_{Bethe}(P) = \Phi^*_{Bethe}(\nu)$ where $\nu_{i \to j} = \tanh(\lambda'_{i \to j})$ is a fixed point of the BP equations. Furthermore, recall from \eqref{eqn:dual-gradient} that
\[ \partial_{\nu_{i \to j}} \Phi^*_{Bethe}(\nu) = \frac{1}{\nu_{i \to j} + 1/(\theta_{ij} \phi(\nu)_{j \to i})} - \frac{1}{\nu_{i \to j} + 1/(\theta_{ij} \nu_{j \to i})}. \]
Based on this we claim that for $\mu_{i \to j} = |\nu_{i \to j}|$, $\Phi^*_{Bethe}(\mu) \ge \Phi^*_{Bethe}(\nu)$. We consider flipping one negative coordinate $\nu_{i \to j}$ to 
$|\nu_{i \to j}|$ at a time and show $\Phi^*_{Bethe}$ is non-decreasing under this operation. First we compute the change using \eqref{eqn:dual-gradient}:
\begin{align*}
 \Phi^*_{Bethe}(\nu_{\sim(i \to j)}, |\nu_{i \to j}|) - \Phi^*_{Bethe}(\nu) 
 &= \int_{-|\nu_{i \to j}|}^{|\nu_{i \to j}|} \frac{\partial \Phi^*_{Bethe}}{\partial \nu_{i \to j}} d\nu_{i\to j}  \\
 &=\log \frac{|\nu_{i \to j}| + 1/(\theta_{ij} \phi(\nu)_{j \to i})}{-|\nu_{i \to j}| + 1/(\theta_{ij} \phi(\nu)_{j \to i})} - \log \frac{|\nu_{i \to j}| + 1/(\theta_{ij} \nu_{j \to i})}{-|\nu_{i \to j}| + 1/(\theta_{ij} \nu_{j \to i})} \\
&= \log \frac{1 + |\nu_{i \to j}| \theta_{ij} \phi(\nu)_{j \to i}}{1 - |\nu_{i \to j}| \theta_{ij} \phi(\nu)_{j \to i}} - \log \frac{1 + |\nu_{i \to j}| \theta_{ij} \nu_{j \to i}}{1 - |\nu_{i \to j}| \theta_{ij} \nu_{j \to i}}.
\end{align*} 
Finally, we notice that this expression is nonnegative as long as $\phi(\nu)_{j \to i} \ge \nu_{j \to i} \ge 0$. Recall that if we are flipping $\nu_{i \to j}$ from negative to positive, by our previous argument it is guaranteed that $\nu_{j \to i} \ge 0$. Furthermore, initially we start from a BP fixed point so $\phi(\nu)_{j \to i} = \nu_{j \to i}$, and increasing coordinates of $\nu$ only increases $\phi(\nu)$, so we maintain the invariant $\phi(\nu)_{j \to i} \ge \nu_{j \to i}$ for all $j,i$ except possibly for those $\nu_{j \to i}$ which have been previously flipped, in which case there is no issue because we will never flip $\nu_{i \to j}$.

Therefore $\mu$ indeed satisfies that $\Phi^*_{Bethe}(\mu) \ge \Phi^*_{Bethe}(\nu)$, and also from the definition we see that $\mu'_{i \to j} \ge |\phi(\nu)_{i \to j}| = |\nu_{i \to j}| = \mu_{i \to j}$ so $\mu \in S_{post}$.
Therefore by Lemma~\ref{lem:nu*} we see
$\Phi^*_{Bethe}(\mu^*) \ge \Phi^*_{Bethe}(\mu) \ge \Phi^*_{Bethe}(\nu)$. Hence $\mu^*$ is a BP fixed point which corresponds to a locally consistent distribution $Q$ with $\Phi_{Bethe}(Q) = \Phi^*_{Bethe}(\mu^*) \ge \Phi^*_{Bethe}(\nu) = \Phi_{Bethe}(P)$, so $Q$ is a global maximizer of $\Phi_{Bethe}$, and $\mu^*_{i \to j} \ge 0$ for all $i$ and $j$.  
\end{proof}

\begin{theorem}\label{thm:bp-landscape}
The maximal fixed point $\nu^*$ (as defined in Lemma~\ref{lem:nu*}) corresponds to the global maximizer
of the Bethe free energy.

\end{theorem}
\begin{proof}
By Lemma~\ref{lem:global-maximizer} there exists some $\mu$ with $\mu_{i \to j} \ge 0$ for all $i,j$ such that $\Phi_{Bethe}^*(\mu)$ equals the global maximum of the Bethe free energy. However, by Lemma~\ref{lem:nu*}, the fixed point $\nu^*$ satisfies $\Phi_{Bethe}^*(\nu^*) \ge \Phi_{Bethe}^*(\mu)$. Therefore the locally consistent distribution $P$ corresponding to $\nu^*$ (which satisfies $\Phi_{Bethe}^*(\nu^*) = \Phi_{Bethe}(\nu)$) must be a global maximizer of the Bethe free energy.
\end{proof}
\subsection{Convergence rate of belief propagation}

\emph{A priori}, there is no significance to the value of $\Phi^*_{Bethe}$ on a general (non-fixed point) $\nu$. However, we observe that $\Phi^*_{Bethe}$ behaves nicely with respect to BP for messages in $S_{pre}$: 
\begin{lemma}\label{lem:pre-monotone}
Suppose that $\nu^{(0)} \in S_{pre}$ and define the BP iterates $\nu^{(t + 1)} := \phi(\nu^{(t)})$. Then for any $T \ge 0$ and any $\mu$ such that $\nu^{(T)} \le \mu \le \nu^{(0)}$
it follows that 
\[ \Phi^*_{Bethe}(\mu) \le \Phi^*_{Bethe}(\nu^{(T)}). \]
In particular,
\[ \Phi^*_{Bethe}(\nu^{(0)}) \le \Phi^*_{Bethe}(\nu^{(T)}). \]
\end{lemma}
\begin{proof}
We prove this by constructing a coordinate-wise monotonically decreasing path from $\mu$ to $\nu^{(T)}$ contained in $S_{pre}$. Then, because the gradient is coordinate-wise nonpositive in $S_{pre}$ due to Lemma~\ref{lem:dual-gradient-properties}, it follows that the first derivative of $\Phi^*$ along (each segment of) this path is nonnegative which proves the inequality by integration.

We construct this path segment-by-segment using an iterative process. For any $\nu$ such that $\nu^{(T)} \le \nu \le \nu^{(0)}$, define $T(\nu,i,j) := \max \{t \ge 0 : \nu_{i \to j} \le \nu^{(t)}_{i \to j} \}$.
\begin{enumerate}
    \item Let $\mu(0) = \mu$ and set $s := 0$.
    \item While there exists $i$ and $j$ such that $T(\mu(s),i,j) < T$:
    \begin{enumerate}
        \item Choose $i$ and $j$ which minimize $t := T(\mu(s),i,j)$.
        \item Define $\mu(s + 1)_{i \to j} = \nu^{(t + 1)}_{i \to j}$ and $\mu(s + 1) = \mu(s)$ in all other coordinates. For $s' \in (s,s + 1)$ define $\mu(s')$ by linearly interpolating $\mu(s)$ and $\mu(s + 1)$.
        \item Set $s := s + 1$.
    \end{enumerate}
\end{enumerate}
Note that this process maintains the invariant $\mu(s) \ge \nu^{(t)}$ and that at each step of the above process, we increase $T(\mu(s),i,j)$ by 1 so the process must terminate in a finite number of steps with the path $\mu(\cdot)$ terminating at $\nu^{(T)}$. It remains to check that this process stays inside of $S_{pre}$ which we check by induction. Given that $\mu(s) \in S_{pre}$, let $t$ be as defined in step 2 (a) above and let $\mu'$ be any linear interpolate between $\mu(s)$ and $\mu(s + 1)$. Then we know $\mu(s) \le \nu^{(t)}$ so $\phi(\mu(s)) \le \phi(\nu^{(t)}) = \nu^{(t + 1)}$ hence $\phi(\mu')_{i \to j} \le \phi(\mu(s))_{i \to j} \le \nu^{(t + 1)}_{i \to j} =  \mu(s + 1)_{i \to j} \le \mu'_{i \to j}$. For the other coordinates $a,b$ it's immediate from monotonicity and the induction hypothesis that $\phi(\mu')_{a \to b} \le \phi(\mu(s))_{a \to b} \le \mu(s)_{a \to b} = \mu'_{a \to b}$ so $\mu' \in S_{pre}$.
\end{proof}

In order to give a quantitative bound on the convergence of BP, we are faced with an important conceptual difficulty: the BP messages may not converge quickly in parameter space, but if the BP messages are far from a fixed point in parameter space it is hard to show anything about the quality of their estimate to the free energy. We will show how to overcome this difficulty by relating the behavior of BP at zero external field and with additional positive external field, which allows us to take advantage of the smoothness of the primal objective $\Phi_{Bethe}$.

The following Lemma gives the bound (in parameter space) for BP at positive external field which we will use; it is a variant of Lemma 4.3 from \cite{DemboMontanari:10} which is more optimized for our setting. 
It is convenient to rephrase the result of Ising models on trees, in which case it gives a quantitative bound showing that under positive external field, the root marginal on an infinite tree does not distinguish between all-plus and free boundary conditions. The connection to our problem is that, because BP computes exact marginals on trees, Loopy BP computes true marginals on its corresponding ``computation tree'' which is a truncation of the non-backtracking walk tree, with boundary conditions on the bottom level determined by its initialization. 
\begin{lemma}\label{lem:positive-h-convergence}
Suppose $T$ is an infinite tree rooted at $\rho$ and suppose the minimum external field is $h_{min} := \min_i h_i > 0$. Then
\[ \E_{T(\ell)}[X_{\rho} | X_{T(\ell)} = 1] - \E_{T(\ell)}[X_{\rho}] \le \frac{1 + \|J\|_{\infty}}{\ell \tanh(h_{min})} \]
where $\E_{T(\ell)}$ denotes the expectation under the measure where the tree is truncated at level $\ell$.
\end{lemma}
\begin{proof}
Observe that $\E_{T(\ell)}[X_{\rho} | X_{T(\ell)} = 1]$ is the same as $\E_{T(\ell - 1), B}[X_{\rho}]$ where $B$ corresponds to additional external field $\sum_{j \in C(i)} J_{ij}$ at every node $i$ on level $\ell - 1$. Similarly, $\E_{T(\ell)}[X_{\rho}]$ is the same as $\E_{T(\ell - 1),h}[X_{\rho}]$ where there is additional field $B'_i$ of $\sum_{j \in C(i)} \tanh^{-1}(\tanh(J_{ij}) \tanh(h_j))$ at the leaves of $T(\ell - 1)$. Define 
\[ M := \sup_{i \in T, j \in C(i)} \frac{ J_{ij}}{\tanh^{-1}(\tanh(J_{ij}) \tanh(h_j))}
\le \sup_{i \in T, j \in C(i)} \frac{ J_{ij}}{\tanh(J_{ij}) \tanh(h_{min})} \le \frac{1 + \|J\|_{\infty}}{\tanh(h_{min})} \]
where we used the inequalities $\tanh^{-1}(x) \ge x$ and $x/\tanh(x) \le 1 + x$ for $x \ge 0$.

Note from above that $\E_{T(\ell)}[X_{\rho}] = \E_{T(\ell - 1),h}[X_{\rho}] \ge \E_{T(\ell - 1)}[X_{\rho}]$ by Griffith's inequality. Therefore we find 
\begin{align*}
\E_{T(\ell)}[X_{\rho} | X_{T(\ell)} = 1] - \E_{T(\ell)}[X_{\rho}]
   &\le
   \E_{T(\ell)}[X_{\rho} | X_{T(\ell)} = 1] - \E_{T(\ell - 1)}[X_{\rho}]
   = \E_{T(\ell - 1), B}[X_{\rho}] - \E_{T(\ell - 1)}[X_{\rho}] \\
   &\le \E_{T(\ell - 1), MB'}[X_{\rho}] - \E_{T(\ell - 1)}[X_{\rho}] \\
   &\le M(\E_{T(\ell - 1), B'}[X_{\rho}] - \E_{T(\ell - 1)}[X_{\rho}])]
   = M(\E_{T(\ell)}[X_{\rho}] - \E_{T(\ell - 1)}[X_{\rho}])]
\end{align*} 
where the last two inequalities were by Griffith's inequality\footnote{This states the root marginal is monotone in the external fields. As with concavity, this can be seen on the tree by writing the root marginal in terms of the BP recursion and using Lemma~\ref{lem:monotone-concave}.} (using that $B \le MB'$) and by concavity of the root marginal w.r.t. the external field along the line from $0$ to $B'$, which follows\footnote{Alternatively, this can be proved from the GHS inequality, generalizing the argument in \cite{DemboMontanari:10}} from the fact that the marginal at the root can be computed via the BP recursion and that this recursion is a composition of concave and monotone functions due to Lemma~\ref{lem:monotone-concave}, hence itself concave.

Summing the corresponding inequality for levels $k = 1$ to $\ell$ we find
\[ \ell(\E_{T(\ell)}[X_{\rho} | X_{T(\ell)} = 1] - \E_{T(\ell)}[X_{\rho}]) \le \sum_{k = 1}^{\ell}(\E_{T(k)}[X_{\rho} | X_{T(k)} = 1] - \E_{T(k)}[X_{\rho}]) \le M \]
which gives the result.
\end{proof}
We are now ready to prove the main theorem, Theorem~\ref{thm:bp}. For the reader's convenience we recall the statement below.
\begin{theorem}[Restatement of Theorem~\ref{thm:bp}]\label{thm:bp-full}
Initializing $\nu^{(0)}_{i \to j} = 1$ for all $i \sim j$ and performing $t$ steps of the BP iteration on a graph with $m$ edges we have
\[ \Phi^*_{Bethe}(\nu^{(t)}) \le \Phi^*_{Bethe}(\nu^*) \le \Phi^*_{Bethe}(\nu^{(t)}) +  \sqrt{\frac{8mn(1 + \|J\|_{\infty})}{t}} \]
\end{theorem}
\begin{proof}
Observe the lower bound is immediate from Lemma~\ref{lem:pre-monotone} and the definition of $\nu^*$ as the limit of the iterates from $\nu^{(0)}$. It remains to prove the upper bound.

Let $B > 0$ be arbitrary and define $\nu^*(B)$ to be the optimal BP fixed point when the external field everywhere is increased by $B$. Then clearly (from the primal definition of Bethe free energy)
\[ \Phi^*_{Bethe}(\nu^*) + \sum_i B \E_{\nu_*}[X_i] \le \Phi^*_{Bethe}(\nu^*(h)) + \sum_i h \E_{\nu_*(h)}[X_i]. \]
so
\[ \Phi^*_{Bethe}(\nu^*) \le \Phi^*_{Bethe}(\nu^*(h)) + Bn. \]
Define $\nu^{(t)}(B)$ to be the result of the BP iteration after $t$ steps from $\nu^{(0)}$ after having increased the external field at every node by $B$. Observe by monotonicity that $\nu^{(t)} \le \nu^{(t)}(B)$.

Now we appeal to Lemma~\ref{lem:dual-gradient-properties} and Lemma~\ref{lem:positive-h-convergence} to see that 
\[  \Phi^*(\nu^*(B)) - \Phi^*(\nu^{(t)}(B)) \le \|\nu^{(t)}(B) - \nu^*(B)\|_1 \le 2m(1 + \|J\|_{\infty})/Bt \]
(using that $\nu^{(t)}(B)$ is sandwiched between the output of BP initialized from 0 and from all-$1$ with additional external field $B$ to get the last inequality from Lemma~\ref{lem:positive-h-convergence})
hence
\[ \Phi^*(\nu^*) -\Phi^*(\nu^{(t)}(B)) \le hn + 2m(1 + \|J\|_{\infty})/Bt. \]
By Lemma~\ref{lem:pre-monotone} this implies that
\[ \Phi^*(\nu^*) -\Phi^*(\nu^{(t)}) \le hn + 2m(1 + \|J\|_{\infty})/Bt \]
as well since $\nu^{(t)} \le \nu^{(t)}(B) \le \nu^{(0)}$.
Finally we optimize the choice of $h$: solving $hn = 2m/ht$ we find $h^2 = 2m(1 + \|J\|_{\infty})/nt$ so the final upper bound on the excess error is at most $2n\sqrt{2m(1 + \|J\|_{\infty})/nt}$. 
\end{proof}
\subsection{Some Examples}\label{sec:examples}
The previous analysis shows how to compute the Bethe free energy by using a small number of rounds of BP to find approximate maximizer of $\Phi^*_{Bethe}$ (on $S$). However, these messages may be far in \emph{parameter space} (e.g. $\ell_{1}$-distance) from the optimal BP fixed point due to flat directions in the objective. In fact, simple examples show that the number of iterations to reach $o(n)$ distance in parameter space may be \emph{exponential} in $\beta$. These examples also show lower bounds on how quickly the BP estimate for the free energy can converge.
\begin{example}\label{ex:line-bp}
Consider the Ising model on the cycle with fixed edge weight $\beta$. Starting from the all-ones initialization, the messages output at time $t$ are all equal to $\tanh(\beta)^t$, so they converge to 0 as $t \to \infty$. Since $1 - \tanh(\beta) = \frac{2 e^{-\beta}}{e^{\beta} + e^{-\beta}} = O(e^{-2\beta})$ we see it takes $\Omega(e^{2 \beta})$ iterations for the messages to go below $1/2$. 

We also see the that 
\begin{align*} 
\Phi^*_{Bethe}(\nu^{(t)}) 
&= n\log\left[(1 + \tanh(\beta)^{t + 1})^2 + (1 - \tanh(\beta)^{t + 1})^2\right] + n\log \frac{e^{\beta} + e^{-\beta}}{2} - n[\log(1 + \tanh(\beta)^{2t + 1})] \\
&= n \log \left[\frac{2 + 2\tanh(\beta)^{2t + 2}}{1 + \tanh(\beta)^{2t + 1}}\right] + n \log \frac{e^{\beta} + e^{-\beta}}{2} \\
&= n\log(2) + n \log \frac{e^{\beta} + e^{-\beta}}{2} + n\log(1 + \frac{\tanh(\beta)^{2t + 2} - \tanh(\beta)^{2t + 1}}{1 + \tanh(\beta)^{2t + 1}}) \\
&= n\log(2) + n \log \frac{e^{\beta} + e^{-\beta}}{2} + \Theta(ne^{-2\beta} \tanh(\beta)^{2t + 1}).
\end{align*}
Therefore if we want to achieve $\epsilon n$ error in $\Phi^*_{Bethe}$ for $\epsilon < e^{-2 \beta}$, at least $\Omega(e^{\beta})$ iterations of BP are required. 
\end{example}
\begin{example}\label{example:d-regular-bp}
Consider a $d$-regular graph on $n$ nodes with fixed edge weight $\beta$ and no external field. Let $\beta$ be the critical value given by solving $(d - 1)\tanh(\beta) = 1$. Then using symmetry to reduce to a 1-dimensional recursion as before, we see that the BP iteration behaves locally like $x \mapsto x - cx^3$ near the fixed point $x = 0$, where $c = c(d) > 0$.
Solving this recurrence, we see that BP converges in parameter space (in $\ell_{\infty}$ norm) at a $\Theta(1/\sqrt{t})$ rate and the objective value (i.e. estimate for the Bethe free energy density) converges at a $\Theta(1/t^2)$ rate asympotically. 
\end{example}
Furthermore for fixed $\beta$, it's impossible for the convergence rate in parameter space under $\ell_{\infty}$ norm to be dimension-independent, whereas when $h > 0$ we knew this was indeed true by Lemma~\ref{lem:positive-h-convergence}:
\begin{example}
Consider the Ising model on the $2$-ary tree for $\beta$ sufficiently large (past the percolation threshold), so the expectation of the root under a boundary condition of all-ones is strictly greater than $0$. If the depth of the tree is $k = \Theta(\log n)$, then after $k - 1$ rounds of BP initialized from $\vec{1}$ the message from the root to its immediate children will be bounded below by a constant, but after $k$ rounds it will be 0.
\end{example}
Finally, to illustrate the behavior of BP highlighted by our results, we ran the mean-field iteration and BP on a simple $40 \times 40$ square grid example with external field at a single node; the results are shown in Figure~\ref{fig:simple-example}. As shown, a small number of iterations already suffices to get a good estimate of the mean-field and Bethe free energies; as shown
on the log-log plot (Figure~\ref{fig:simple-example} (b)), the convergence rate is consistent with a power law decay as shown in Theorem~\ref{thm:bp}, although with a better exponent than the worst-case bound shows. This is expected as we expect this model to behave similarly to Example~\ref{example:d-regular-bp}; we chose $\beta$ based on the critical value for the $4$-regular tree. In this example, the mean-field iteration converged even faster; again this is consistent with what one would guess based on the behavior in Example~\ref{example:d-regular-mf}, where one observes that away from the critical $\beta$ the mean field iteration converges faster, at an exponential rate asymptotically.

We also see in Figure~\ref{fig:simple-example-2} the importance of initializing from all-ones; the model is the same as before except that $\beta$ is larger, so that long-range behavior can affect BP. In simple examples like this, BP and mean-field iteration will require at least on the order of the diameter many steps in order to converge if started from all-zeros.
\begin{figure}
    \centering
    \begin{subfigure}{.49\linewidth}
    \centering
    \includegraphics[scale=0.45]{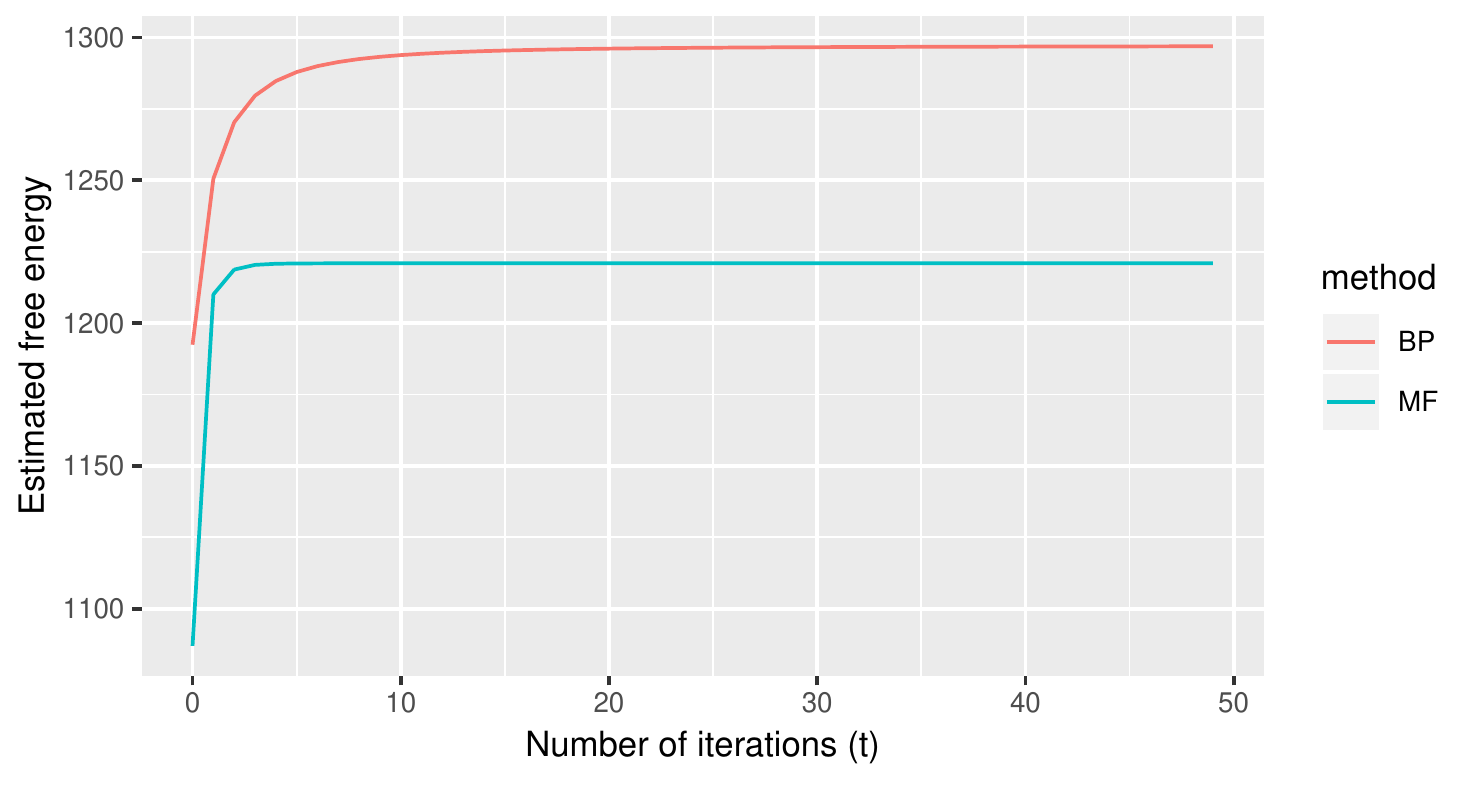}
    \caption{Estimated free energy vs. iteration}
    \end{subfigure}
   \begin{subfigure}{.49\linewidth}
   \centering
   \includegraphics[scale=0.45]{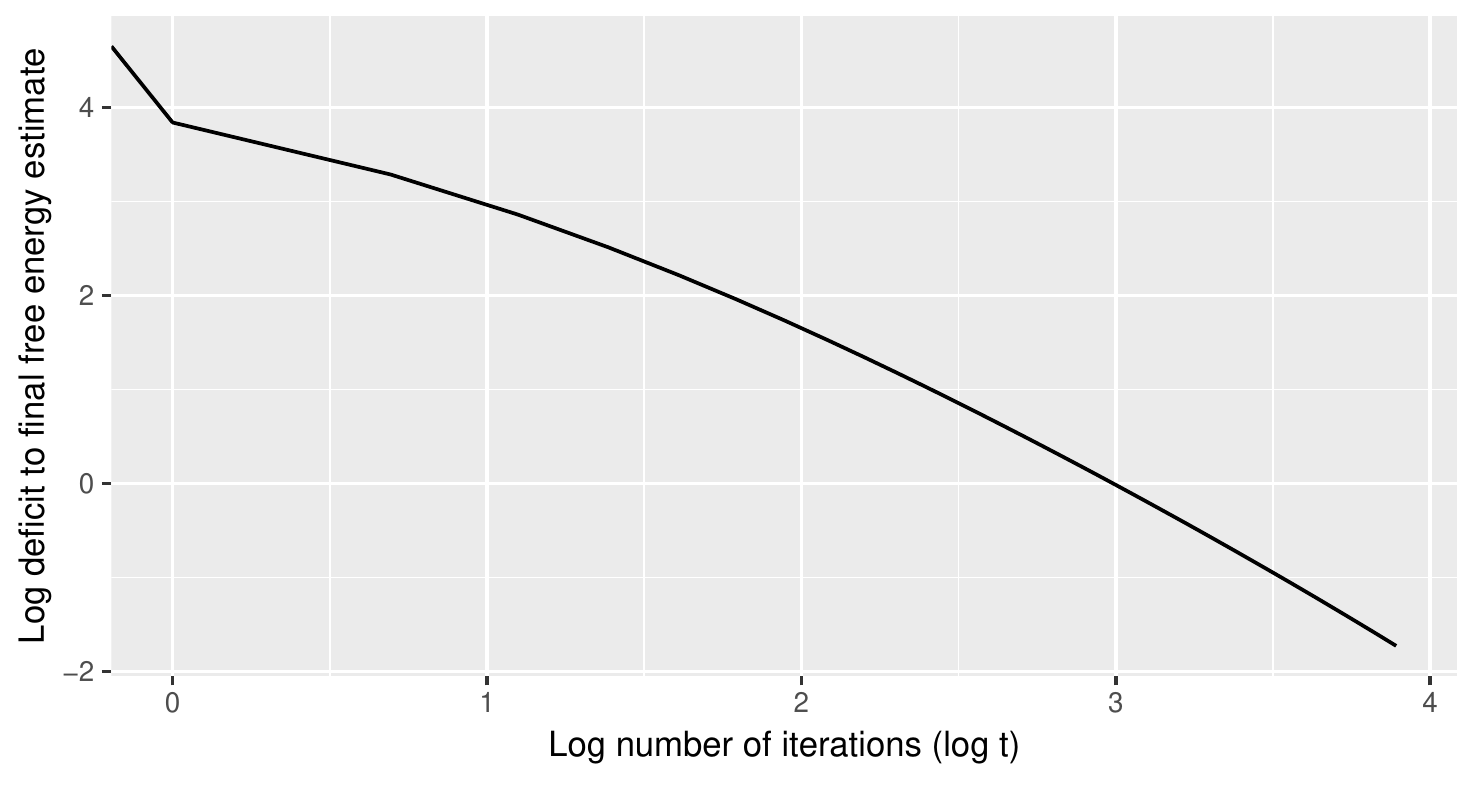}
   \caption{Log iterations vs. log residual error for BP}
   \end{subfigure}
    \caption{Results of mean-field iteration and BP from all-ones initialization on a $40 \times 40$ square grid with edge weights $\beta = \tanh^{-1}(1/3) \approx 0.347$, with zero external field except for external field of strength $5$ at the bottom-left node. In (b) we plot the difference between the true Bethe free energy and the estimated free energy of BP at each iteration on a log-log plot.}
    \label{fig:simple-example}
\end{figure}
\begin{figure}
    \centering
    \includegraphics[scale=0.6]{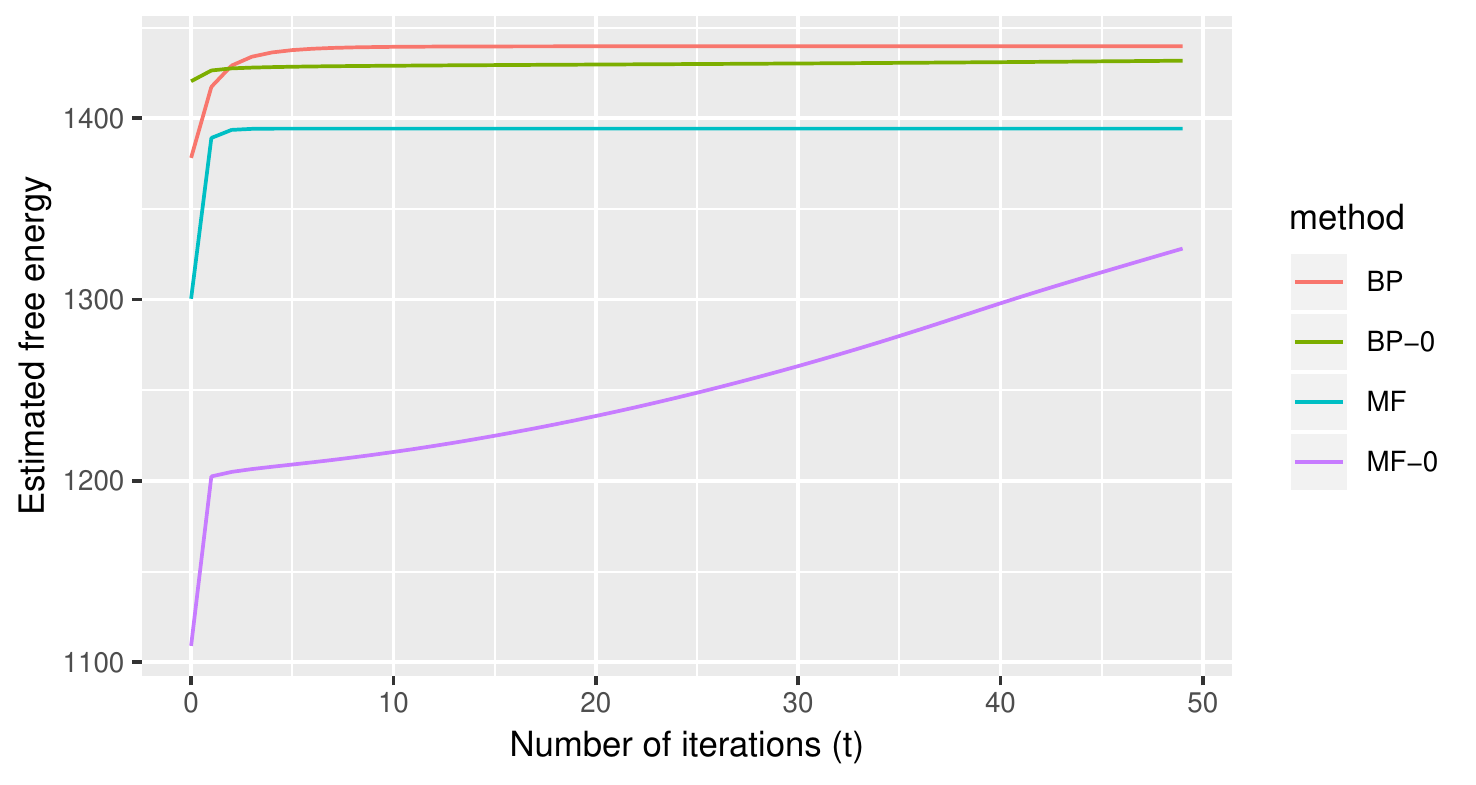}
    \caption{Comparison of BP and mean-field iteration at all-ones initializes (MF,BP) vs. all-zeros initialization (MF-0,BP-0). The instance is the same as in Figure~\ref{fig:simple-example} except that $\beta \approx 0.384$; we see that all-ones initialization leads to quick convergence (consistence with Theorems~\ref{thm:mean-field-convergence} and \ref{thm:bp}) whereas with all-zeros it does not.}
    \label{fig:simple-example-2}
\end{figure}
\section{Computing exponentially good BP messages in polynomial time}\label{sec:exp-convergence}
The bound we proved for BP in Theorem~\ref{thm:bp} showed that if we want to achieve $\epsilon n$ error then $poly(1/\epsilon)$ steps of BP suffice. What if $\epsilon$ is exponentially small? Can a small number of steps of BP achieve $\epsilon$ error?
It turns out the answer to this is negative.
In Example~\ref{ex:line-bp} from the previous section (Ising model on a line at inverse temperature $\beta$), we saw that BP cannot achieve error $O(e^{-2\beta})$ approximating the free energy without taking at least $\Omega(e^{\beta})$ many steps. Therefore, if we want to estimate $\Phi^*_{BP}$ within an exponentially small error in polynomial time we must use a different algorithm.

In this section we use the ideas developed while analyzing BP to give such algorithm, with runtime 
$poly(\log(1/\epsilon),n)$ . In order to achieve asymptotically fast convergence to the optimum, we use tools from convex optimization instead of message-passing. 
Recall that
\[ S_{post} = \{ \nu \ge 0 : \phi(\nu)_{i \to j} \ge \nu_{i \to j} \} \]
is a convex set (with an obvious separation oracle) and observe that by Theorem~\ref{thm:bp-landscape}, the optimum $\nu^*$ is the maximizer of the following convex program:
\begin{equation}\label{eqn:convex-program}
\nu^* = \argmax_{\nu \in S_{post}} \sum_{i,j} \nu_{i \to j}. 
\end{equation}
We show how to compute the maximizer using the ellipsoid method (although a wide variety of methods are applicable, see e.g. \cite{bubeck2015convex}). First we analyze the case where $h$ is bounded below, and we show the dependence on $h_{min}$ is very benign.
\begin{lemma}\label{lem:ellipsoid-positive}
Suppose that $h_{min} := \min h_i > 0$. Then given $\epsilon > 0$, the ellipsoid method applied to \eqref{eqn:convex-program} computes $\nu \in S'$ such that
\[ \|\nu^* - \nu\|_1 \le \epsilon \]
after $O(m^2 (\log(n/h_{min}) + \log(1/\epsilon)))$ steps of the ellipsoid method.
\end{lemma}
\begin{proof}
Recall from Theorem 2.4 of \cite{bubeck2015convex} that the feasible set $S$ contains a ball of radius $r$ and is contained in a ball of radius $R$, then for any convex function $f : \mathbb{R}^{d} \to [-B,B]$ and $x_t$ the result of $t$ steps of ellipsoid method, satisfies
\[ \max_{x \in S} f(x) - f(x_t) \le \frac{2 B R}{r} e^{-t/2d^2} \]
as long as $t \ge 2d^2\log(R/r)$.
Note that our function of interest $\sum_{i,j} \nu_{i \to j}$ is bounded with $B = 2m$ and is contained in $[0,1]^n \subset \mathcal{B}(0,2\sqrt{n})$. By assumption we see that $S_{post}$ contains $[0,\tanh(h_{min})]^{2m}$ so it contains a ball of radius $\frac{\tanh(h_{min})}{2}$.
\end{proof}
In order to guarantee the optimization problem is well-behaved, we perturb it by a tiny amount, and this gives an algorithm for the general case. (If we do not perturb the model by adding a tiny external field, $S_{post}$ may be a lower-dimensional, measure-zero set which would be problematic for the ellipsoid method.)
\begin{theorem}\label{thm:bethe-ellipsoid}
Suppose $\epsilon > 0$. There is an algorithm which runs in time $poly(n, \log(1/\epsilon))$  and returns $\nu$ such that
\[ \left|\Phi^*_{Bethe}(\nu^*) - \Phi^*_{Bethe}(\nu)\right| \le \epsilon \]
\end{theorem}
\begin{proof}
Fix $B > 0$ to be optimized later. We add external field $B$ everywhere in the model and apply Lemma~\ref{lem:ellipsoid-positive} to see that we can compute $\nu$ such that
\[ \|\nu^*(B) - \nu\|_1 \le \epsilon/2 \]
in time $poly(\log(1/\epsilon),n,\log(1/B))$. Then we see by the same argument as in Theorem~\ref{thm:bp} that 
\[ |\Phi^*_{Bethe}(\nu^*) - \Phi^*_{Bethe}(\nu)| \le B m + \epsilon/2. \]
Finally taking $B = \epsilon/2m$ shows the result.
\end{proof}
The same approach works for the mean-field problem as well:
\begin{theorem}
Suppose $\epsilon > 0$. There is an algorithm which runs in time $poly(n, \log(1/\epsilon))$  and returns $x$ such that
\[ \Phi_{MF}(x) - \Phi_{MF}(x^*) \le \epsilon \]
\end{theorem}
\begin{proof}
Recall the definition of the convex set $S_{pre} := \{x \ge 0 : \tanh^{\otimes n}(J x + h) \le x \}$ and consider the optimization problem
\[ \max_{x \in S_{pre}} \sum_i x_i. \]
Then we do everything the same way as in the proof of Theorem~\ref{thm:bethe-ellipsoid}: the algorithm proceeds perturbing the problem by adding a tiny external field $B = \epsilon/2$, and then solving it with ellipsoid method.
\end{proof}

\noindent
\textbf{Acknowledgements: } The author thanks Vishesh Jain for helpful discussions and for suggesting the argument in Section~\ref{sec:asymptotic-improvement},
Elchanan Mossel, Nike Sun, Matthew Brennan, and Enric Boix for valuable discussions related to this work, and Ankur Moitra and Andrej Risteski for useful discussions on related topics.
\bibliographystyle{plain}
\bibliography{all,ising-regularity}
\appendix

\section{Background: BP and the Bethe Free Energy (appendix)}
\label{apdx:background-bethe-deferred}
In this appendix corresponding to Section~\ref{sec:background-bethe}, we quickly recall the basic definitions,  facts, and notations involving the Bethe free energy, its `dual' formulation, and the connection to Belief Propagation -- as described in \cite{YeFrWe:03} and reference text \cite{MezMon:09}. We repeat some of the standard calculations in order to express the results in our variables $\nu_{i \to j}$.
The Lagrangian corresponding to the optimization problem \eqref{eqn:bethe-functional} over the polytope of locally consistent distributions (which is defined over all, not necessarily consistent, $P_{ij}$ and $P_i$) is
\[ \mathcal{L}(P,\lambda) = \Phi_{Bethe}(P) + \sum_{i,j,x_i} \lambda_{i \to j}(x_i)(\sum_{x_j} P_{ij}(x_i,x_j) - P_i(x_i)) + \sum_i \lambda_i(\sum_{x_i} P_i(x_i) - 1) \]
where we ignore the constraint $P_i(x_i) \ge 0$ because, given the other constraints, this constraint is always satisfied at a critical point (since the derivative of $H(Ber(p))$ diverges as $p \to 0$ or $p \to 1$).

By differentiating with respect to $P$, and setting $\lambda'_{i \to j} = \frac{\lambda_{i \to j}(1) - \lambda_{i \to j}(-1)}{2}$, one finds that at a critical point of the Lagrangian  that
\[ P_{ij}(x_i,x_j) \propto e^{J_{ij} x_i x_j + \lambda_{i \to j}(x_i) + \lambda_{j \to i}(x_j)} \propto e^{J_{ij} x_i x_j + \lambda'_{i \to j} x_i + \lambda'_{j \to i} x_j} \]
and 
\[ P_i(x_i) \propto \exp\left(\frac{1}{\deg(i) - 1} \sum_j \lambda_{i \to j}(x_i) - \frac{h_i}{\deg(i) - 1} x_i\right) \propto \exp\left(\frac{1}{\deg(i) - 1} \sum_j \lambda'_{i \to j} x_i - \frac{h_i}{\deg(i) - 1} x_i\right). \]
Furthermore by differentiating with respect to $\lambda$ we see that the constraints are satisfied, therefore for any $i \sim j$ that
$P_i(x_i) = \sum_{x_j} P_{ij}(x_i, x_j)$ hence
\[ P_i(x_i)^{\deg(i) - 1} \propto \prod_{k \in \partial i \setminus \{j\}} \sum_{x_k} P_{ik}(x_i,x_k) \propto \sum_{x_{\partial i \setminus j}} e^{\sum_k (J_{ik} x_i x_k + \lambda'_{i \to k} x_i + \lambda'_{k \to i} x_k)} = e^{\sum_k \lambda'_{i \to k} x_i} \sum_{x_{\partial i \setminus j}} e^{\sum_k (J_{ik} x_i + \lambda'_{k \to i}) x_k}  \]
so
\[ e^{\lambda'_{i \to j}x_i - h_i x_i} \propto \sum_{x_{\partial i \setminus j}} e^{\sum_k (J_{ik} x_i + \lambda'_{k \to i}) x_k} \propto \prod_k \sum_{x_k} e^{J_{ik} x_i} e^{\lambda'_{k \to i} x_k}. \]
Define $\nu_{i \to j} := \tanh(\lambda'_{i \to j})$ so $\frac{1 + \nu_{i \to j} x_i}{2} = \frac{e^{\lambda'_{i \to j}x_i}}{e^{\lambda'_{i \to j}} + e^{-\lambda'_{i \to j}}}$, then we see
\begin{align*} 
\nu_{i \to j} 
&= \frac{e^{h_i} \prod_k \sum_{x_k} e^{J_{ik} x_k} e^{\lambda'_{k \to i} x_k} - e^{-h_i} \prod_k \sum_{x_k} e^{-J_{ik} x_k} e^{\lambda'_{k \to i} x_k}}{e^{h_i} \prod_k \sum_{x_k} e^{J_{ik} x_k} e^{\lambda'_{k \to i} x_k} + e^{-h_i} \prod_k \sum_{x_k} e^{-J_{ik} x_k} e^{\lambda'_{k \to i} x_k}} \\
&=\frac{e^{h_i} \prod_k \sum_{x_k} e^{J_{ik} x_k} (1 + \nu_{k \to i} x_k) - e^{-h_i} \prod_k \sum_{x_k} e^{-J_{ik} x_k} (1 - \nu_{k \to i} x_k)}{e^{h_i} \prod_k \sum_{x_k} e^{J_{ik} x_k} (1 + \nu_{k \to i} x_k) + e^{-h_i} \prod_k \sum_{x_k} e^{-J_{ik} x_k} (1 - \nu_{k \to i} x_k)} \\
&= \tanh(h_i + \sum_{k \in \partial i \setminus j} \tanh^{-1}(\tanh(J_{ik}) \nu_{k \to i})) 
\end{align*}
which is the form of the BP equation we will typically refer to. We will denote the right hand side by $\phi(\nu)_{i \to j}$ so the BP iteration is given by $\nu \mapsto \phi(\nu)$. We will also define $\theta_{ik} = \tanh(J_{ik})$.
Finally, we will explicitly rewrite the Bethe free energy at a critical point in terms of the messages $\nu_{i \to j}$. We claim that at a critical point $(P,\lambda)$
\[ \Phi_{Bethe} = \sum_{i} F_i(\lambda) - \sum_{i \sim j} F_{ij}(\lambda) \]
where
\begin{align*}
F_i(\lambda) 
&:= \log \sum_{x_i} e^{h_i x_i}  \prod_{j \in \partial i} \sum_{x_j} e^{J_{ij} x_i x_j} e^{\lambda'_{j \to i} x_j} = \log \sum_{x_i,x_{\partial i}} e^{h_i x_i + \sum_j J_{ij} x_i x_j + \lambda'_{j \to i} x_j} \\
F_{ij}(\lambda) 
&:= \log \sum_{x_i,x_j} e^{J_{ij} x_i x_j + \lambda'_{i \to j} x_i + \lambda'_{j \to i} x_j}.
\end{align*}
To see this observe that from the Gibbs variational principle that (considering a joint distribution where we sample $X_i$ from $P_i$ and then $X_j | X_i$ according to $P_{ij}$)
\begin{align*} 
F_i(\lambda) 
&= \E[h_i X_i + \sum_j J_{ij} X_i X_j + \lambda'_{j \to i} X_j] + H(X_i) + \sum_{j \in \partial i} H(X_j | X_i) \\
&= \E[h_i X_i + \sum_j J_{ij} X_i X_j + \lambda'_{j \to i} X_j] + \sum_{j \in \partial i} H(X_i,X_j) - (\deg(i) - 1) H(X_i)
\end{align*}
and
\[ F_{ij}(\lambda) = \E[\sum_j J_{ij} X_i X_j + \lambda'_{i \to j} X_i + \lambda'_{j \to i} X_j] + H(X_i,X_j) \]
so summing all of these terms does indeed give $\Phi_{Bethe}(P)$. Finally we rewrite in terms of $\nu_{i \to j}$ to get
\begin{align*} 
F_i(\nu) 
&= \log \left[e^{h_i} \prod_{j \in \partial i} \sum_{x_j} e^{J_{ij} x_j} \frac{1 + \nu_{j \to i}x_j}{2}  + e^{-h_i} \prod_{j \in \partial i} \sum_{x_j} e^{-J_{ij} x_j} \frac{1 + \nu_{j \to i}x_j}{2}\right] \\
&= \log \left[e^{h_i} \prod_{j \in \partial i} (1 + \tanh(J_{ij}) \nu_{j \to i})  + e^{-h_i} \prod_{j \in \partial i} (1 - \tanh(J_{ij}) \nu_{j \to i}) \right] + \sum_{j \in \partial i} \log \frac{e^{J_{ij}} + e^{-J_{ij}}}{2}
\end{align*}
and
\begin{align*}
F_{ij}(\nu) 
&= \log \sum_{x_i}\sum_{x_j} e^{J_{ij} x_i x_j}  \frac{1 + \nu_{i \to j} x_i}{2} \frac{1 + \nu_{j \to i} x_j}{2}  \\
&= \log \left(\frac{e^{J_{ij}} + e^{-J_{ij}}}{2} + \frac{(e^{J_{ij}} - e^{-J_{ij}}) \nu_{i \to j} \nu_{j \to i}}{2} \right) \\
&= \log(1 + \tanh(J_{ij}) \nu_{i \to j} \nu_{j \to i}) + \log\left(\frac{e^{J_{ij}} + e^{-J_{ij}}}{2}\right).
\end{align*}

\end{document}